\documentclass[letterpaper, 10 pt, conference]{ieeeconf}

\IEEEoverridecommandlockouts                              

\usepackage{graphicx} 
\usepackage{epsfig} 
\usepackage{amsmath} 
\usepackage{amssymb}  
\usepackage{subfig}
\usepackage[ruled,longend]{algorithm2e}
\usepackage{multicol}
\usepackage{color}
\newtheorem{theorem}{Theorem}
\newtheorem{proposition}{Proposition}
\newtheorem{lemma}{Lemma}
\newtheorem{remark}{Remark}
\newtheorem{corollary}{Corollary}
\newtheorem{assumption}{Assumption}
\DeclareMathOperator*{\argmin}{arg\,min}
\newcommand{\mf}{\mathcal{F}}
\newcommand{\mk}{\mathcal{K}}
\newcommand{\mg}{\mathcal{G}}
\title{On the Convergence of Reinforcement Learning in Nonlinear Continuous State Space Problems}

\author{ Raman Goyal, Suman Chakravorty, Ran Wang, Mohamed Naveed Gul Mohamed
\thanks{The authors are with the Department of Aerospace Engineering, Texas A\&M University, College Station, TX 77843 USA. \{\tt ramaniitrgoyal92, schakrav, rwang0417, naveed\} @tamu.edu}}

\begin{document}

\maketitle  

%

%



\begin{abstract}
We consider the problem of Reinforcement Learning for nonlinear stochastic dynamical systems. We show that in the RL setting, there is an inherent ``Curse of Variance" in addition to Bellman's infamous ``Curse of Dimensionality", in particular, we show that the variance in the solution grows factorial-exponentially in the order of the approximation. A fundamental consequence is that this precludes the search for anything other than ``local" feedback solutions in RL, in order to control the explosive variance growth, and thus, ensure accuracy. We further show that the deterministic optimal control has a perturbation structure, in that the higher order terms do not affect the calculation of lower order terms, which can be utilized in RL to get accurate local solutions.
\end{abstract} 

\section{Introduction}\label{sec1}

A large class of decision making problems under uncertainty can be posed as a nonlinear stochastic optimal control problem that requires the solution of an associated Dynamic Programming (DP) problem, however, as the state dimension increases, the computational complexity
goes up exponentially in the state dimension \cite{bertsekas1}:  the manifestation of Bellman's infamous ``curse of dimensionality (CoD)" \cite{bellman}. To understand the CoD better, consider the simpler problem of
estimating the cost-to-go function of a feedback policy $\mu_t(\cdot)$. 
Let us further assume that the cost-to-go function can be ``linearly parametrized'' as: $J_t^{\mu} (x) = \sum_{i = 1}^M \alpha^i_t \phi_i(x)$, where the $\phi_i(x)$'s are some \emph{a priori} basis functions. Then the problem of estimating $J_t^{\mu} (x)$ becomes that of
estimating the parameters $\bar{\alpha}_t = \{ \alpha^1_t, \cdots, \alpha^M_t \}$.
This can be shown to be the recursive solution of the linear equations
$\bar{\alpha}_t = \bar{c}_t + L_t\bar{\alpha}_{t+1}, \text{ where } \bar{c}_t = [c^i_t],$
with $ c^i_t = \int c(x, \mu_t(x)) \phi_i (x) \,dx,$ 
and 
$
L_t^{ij} = \int \int p^{\mu_t} (x'|x) \phi_i (x') \phi_j(x) \,dx' \,dx, \; i, j \in \{1,\dots,M\},
$ where $p^{\mu_t}(./.)$ is the transition density of the Markov chain under policy $\mu_t$.
This can be done using
numerical quadratures given knowledge
of the model $p^{\mu}(x'|x)$, termed Approximate DP (ADP), or alternatively, in Reinforcement Learning (RL), simulations of the process under the policy $\mu_t$, $x_t \xrightarrow{\mu_t(x_t)} x_{t + 1} \rightarrow \cdots $, is used to get an approximation of the $L^{ij}_t$ by sampling, 
and solve the equation above either batchwise or recursively \cite{parr3,bertsekas1}. 
But, as the dimension $d~$ increases, the number of basis functions and the number of evaluations required to evaluate the integrals go up exponentially. 
There has been recent success using the Deep RL paradigm where deep neural networks are used as nonlinear function approximators to keep the parametrization tractable \cite{RLHD1, RLHD2, RLHD3, RLHD4, RLHD5}, however, the training times required for these approaches is still prohibitive.
Hence, the primary problem with ADP/ RL techniques is the CoD inherent in the complex representation of the cost-to-go function, and the exponentially large number of evaluations required for its estimation. In this paper, we show that there is an additional ``Curse of Variance" that afflicts the RL solution, the fact that the variance grows at a factorial-exponential rate in the order of the approximation, that precludes us from solving for higher order approximations of the feedback law. 
Prior research has focused in some detail on the sample complexity of RL for finite control problems \cite{azar2012sample,kakade2003sample,munos2008finite}, and  the case of optimal Linear Quadratic Control (LQR) in continuous state and control spaces \cite{Recht2019sample,recht2019tour}.
We study the general nonlinear problem, and show the scale of the variance inherent in an RL estimate. 
Albeit anecdotal and empirical evidence of the variance phenomenon has always existed in the RL literature \cite{henderson2018deep}, we believe we are the first to exactly enumerate the factorial-exponential growth and its consequences: it is necessary that we look for local solutions in order to find accurate solutions, that stochastic control problems are fundamentally intractable, and the best we can hope for is a suitably accurate deterministic approximation (see points 1-5 of contributions below). 
However, this does not mean we need to give up on global optimality. In \cite{mohamed2020optimality}, we established the local optimality of the deterministic feedback law, in that the nominal (zero noise) action, and the linear feedback action, of the optimal stochastic and deterministic policies are close to fourth order in a small noise parameter, starting at any given state, which, when allied with replanning, recovers a near-optimal solution. Thus, a local solution allied with replanning is an efficient and near-optimal way to solve nonlinear stochastic control problem rather than solve for a global (higher order) solution. This gets us to the context of Model Predictive Control (MPC).

In the case of continuous state, control and observation space problems, the
MPC \cite{Mayne_1, Mayne_2} approach has been used with a lot of success in the control system and robotics community.   
However, stochastic control problems, and the control of
uncertain systems in general, is still an unresolved problem in MPC. As noted in \cite{Mayne_1}, the problem arises due to
the fact that in stochastic control problems, the MPC optimization at every
time step cannot be over deterministic control sequences, but rather has to be
over feedback policies, which is, in general, intractable. Thus, the tube-based MPC approach, and its stochastic counterparts,
typically consider linear systems \cite{T-MPC1, T-MPC2,T-MPC3} for which a
linear parametrization of the feedback policy suffices but the methods become intractable when dealing with nonlinear systems \cite{Mayne_3}. 
In this paper, we show that the stochastic problem is fundamentally intractable, and if solved in an RL fashion, leads to a very high variance in the solution. In particular, it becomes necessary to look for local deterministic solutions to ensure accuracy (which are locally optimal due to the results of  \cite{mohamed2020optimality}),  and re-plan when necessary as in MPC, to recover global optimality. 

We summarize our contributions as follows.

1. It is fundamentally intractable to solve for a high order approximation of a feedback law for optimal control via RL (global/ nonlocal), since the variance of the solution grows factorial-exponentially in the order of the approximation.\\
2. The deterministic problem has a perturbation structure, in that higher order terms do not affect the calculation of lower order terms, and thus, when a model is known, the calculations can be closed at any order without affecting the accuracy of the lower order terms.\\
3. If the deterministic problem is solved in an RL fashion, then an accurate solution can be found, if and only if we concentrate on a suitably local solution, enforced via constraining the random exploration around a nominal trajectory. \\
4. The stochastic problem is intractable in the sense that it lacks a perturbation structure, and if solved via RL, this necessarily implies very high bias as well as variance, and hence, inaccuracy in the result.\\
5. The perturbation structure and locality of the solution are key to an accurate RL implementation.\\
\vspace{-3mm}
\textbf{Outline of Paper.}
The rest of the document is organized as follows: Section II outlines the Problem Formulation, Section III studies the convergence of Policy evaluation in a finite time RL setting and the resulting variance in the solution. Section IV derives a perturbation structure inherent to the deterministic policy evaluation problem, and shows how to leverage this for accurate local RL solutions. Section V concentrates on the stochastic policy evaluation problem and shows its fundamental intractability in terms of the lack of a perturbation structure, and the resulting high variance for RL solutions. Section VI gives empirical results in a simple example to validate the theoretical development.
\vspace{-2mm}
\section{Problem Formulation}
\label{section:prob}
The problem of control under uncertainty can be formulated as a stochastic optimal control problem in the space of feedback policies. We assume here that the uncertainty in the problem lies in the system's process model. \\
%
\textit{\textbf{{System Model:}}} For a dynamic system, we denote the state and control vectors by $x_t \in \ \mathbb{X} \subset \ \mathbb{R}^{n_x}$ and $u_t \in \ \mathbb{U} \subset \ \mathbb{R}^{n_u}$ respectively at time $t$. The motion model $h : \mathbb{X} \times \mathbb{U} \times \mathbb{R}^{n_u}   \rightarrow \mathbb{X} $ is given by the equation 
\begin{equation}
    x_{t+1}= h(x_t, u_t, w_t); \  w_t \sim \mathcal{N}(0, {\Sigma}_{w_t}) 
    \label{eq:model},
\end{equation}
where \{$w_t$\} are zero mean independent, identically distributed (i.i.d) random sequences with variance ${\Sigma}_{w_t}$.\\

\vspace{-3mm}
\textit{\textbf{Stochastic optimal control problem:}} The stochastic optimal control problem for a dynamic system with initial state $x_0$ is defined as:
\vspace{-1mm}
\begin{equation}
    J_{\pi^{*}}(x_0) = \min_{\pi} \ E \left[{\sum^{T-1}_{t=0} c(x_t, \pi_t (x_t)) + g(x_T)}\right],
\end{equation}
$
   s.t.\ x_{t+1} = h(x_t, \pi_t (x_t), w_t),
$
~where:
the optimization is over feedback policies $\pi := \{ \pi_0, \pi_1, \ldots, \pi_{T-1} \} $ and $\pi_t(\cdot)$: $\mathbb{X} \rightarrow \mathbb{U}$ specifies an action given the state, $u_t = \pi_t(x_t)$;
$J_{\pi^{*}}(\cdot): \mathbb{X} \rightarrow \mathbb{R}$  is the cost function on executing the optimal policy $\pi^{*}$; $c_t(\cdot,\cdot): \mathbb{X} \times \mathbb{U} \rightarrow \mathbb{R} $  is the one-step cost function; $g(\cdot): \mathbb{X} \rightarrow \mathbb{R}$ is the terminal cost function; $T$ is the horizon of the problem.
The solution to the above problem is given by the Dynamic Programming equation:
\begin{equation}
    J_t(x) = \min_u [c(x,u) + E[J_{t+1}(x')]],
\end{equation}
where $x' \sim p(./x,u)$, and $p(./x,u)$ denotes the transition density of the state at the next time step arising form the system dynamics, given the control $u$ is taken at state $x$, solved with the terminal condition $J_T(x) = g(x)$. The DP equation can be solved using the so-called Policy Iteration method, where given a time varying feedback policy $\pi^{(k)}_t(\cdot)$, one first solves for the cost function corresponding to it:
\begin{equation}
J^{(k)}_t (x) = c(x, \pi^{(k)}_t(x)) + E[J_{t+1}^{(k)}(x')],
\end{equation}
where $x'\sim p(./x,\pi^{(k)}_t(x))$, and the above equation is solved with the terminal condition $J_T^{(k)}(x) = g(x)$, which is followed by a policy improvement step: 
\begin{equation}
  \pi^{(k+1)}_t (x) = \argmin_u [c(x,u) + E[J_{t+1}^{(k)}(x')]],
\end{equation}
where $x'\sim p(./x,u)$. This process is followed till convergence, starting with some initial time varying policy $\pi^{(0)}_t(x)$ \cite{bertsekas1}.
%
\vspace{-2mm}
\section{Convergence of Policy Evaluation}
In the following, we shall concentrate on the Policy Evaluation (PE) part of Policy iteration, in particular, a single Policy Evaluation step, to show the convergence issues inherent, and at the end of this section, outline the issues arising from the dynamic recursion. We shall consider a synchronous model of computing, i.e., where all the experiments are done first, and the cost functions at any step updated using all the experiments. We consider the deterministic scalar state case for simplicity, the generalization to the vector state case is straightforward and the stochastic case is treated in Section 5.
Let us rewrite the policy evaluation equation from above as: 
\begin{equation}\label{pol_eval}
    J_t(x_t) = c_t(x_t) + J_{t+1}(f(x_t)),
\end{equation}
where the deterministic dynamics are $x_{t+1} = f(x_t)$, and given that the terminal cost $J_T(x_T) = g(x_T)$. In the context of Policy Iteration, the dynamics corresponds to the closed loop under some feedback policy $\pi(x)$, i.e., $f(x) = h(x, \pi(x), 0)$. Note that the policy is, in general, time varying, and thus, the dynamics should also be time varying. But we consider time invariant dynamics for simplicity, and all the results obtained below generalize to the time varying case in a straightforward fashion.

\paragraph{Computing Model} Suppose that we have basis functions $\{\phi^1(x), \cdots \phi^N(x)\}$ such that any $J_t(x) = \sum_i \alpha_t^i \phi^i(x)$ for suitably chosen coefficients $\alpha_t^i$, and such that $c_t(x) = \sum_i c_t^i \phi^i(x)$. Suppose now that we are given $R$ samples from the dynamical system, say $\{x_1^{(k)}, \cdots x_t^{(k)}\}$, for $k = 1, \cdots R,$ that are sampled from some, in general, time varying density $p_t(\cdot)$. We leave the question of what this density ought to be to later on in our development. Given the samples, we write:
\begin{equation}
    J_t(x^{(k)}_t) = c(x_t^{(k)}) + J_{t+1}(f(x_t^{(k)})) + v^{(k)}_t,
\end{equation}
where $v^{(k)}_t$ is an independent identically distributed (i.i.d.) noise sequence for all time steps $t$. Representing the cost functions in terms of the basis functions, we obtain:
\begin{equation}
\bar{\alpha}_t \phi_t^{(k)} = \bar{c}_t \phi_t^{(k)} + \bar{\alpha}_{t+1} \phi_{t+1}^{(k)} + v^{(k)}_t,
\end{equation}
where $\bar{\alpha}_t^{(k)} = \begin{bmatrix} \alpha_t^{1} \cdots \alpha_T^{N} \end{bmatrix}$, $\bar{c}_t= \begin{bmatrix} c_t^1 \cdots c_t^N \end{bmatrix}$, $\phi_t^{(k)} = \begin{bmatrix} \phi^1 (x_t^{(k)})\\ \vdots\\ \phi^N(x_t^{(k)})\end{bmatrix}$, and $\phi_{t+1}^{(k)} = \begin{bmatrix} \phi^1 (x_{t+1}^{(k)})\\ \vdots\\ \phi^N(x_{t+1}^{(k)})\end{bmatrix}$, where note that $x_{t+1}^{(k)} = f(x_t^{(k)})$. 

\paragraph{RL as Least Squares} Then, we may view the above as the following least squares problem:
\begin{equation}
    \bar{\alpha}_t^R = \argmin_{\bar{\alpha}_t} ||\bar{\alpha_t} \Phi_t^R - \bar{c}_t \Phi_t^R - \bar{\alpha}_{t+1}^R \Phi_{t+1}^R||^2,
\end{equation}
where $\Phi_t^R = \begin{bmatrix} \phi_t^{(1)}, \cdots, \phi_t^{(R)} \end{bmatrix}$, and the supercase $R$ is used to denote the solution after $R$ samples. In this case, we are sweeping back in time starting at the final time $T$, and thus, it is assumed above that we have solved for $\alpha_{t+1}^R$ already. The solution to this problem is standard and given by:
\begin{equation} \label{RL-LS}
    \bar{\alpha}_t^R = \bar{c}_t + \bar{\alpha}_{t+1}^R \Phi_{t+1}^R \Phi_t^{R'} (\Phi_t^R \Phi_t^{R'})^{-1},
\end{equation}
where $A'$ denotes the transpose of a matrix $A$.\\
Now, we shall establish some properties of the least squares (LS) solution above in the context of RL. First, let us find the ``true" solution of the Policy Evaluation equation \eqref{pol_eval}. First, we make the following assumption.\\

\begin{assumption} \label{basis}
There exist a set of constants $\beta_{ij}$, $i= 1, 2, \cdots N$,  and $j = 1, \cdots N'$, where $N' > N$, such that for any $\phi^i(f(x)) = \sum_{j=1}^{N'} \beta^{ij} \phi^j(x)$. 
\end{assumption}
The reason $N'> N$ is that, in general, unless $f(\cdot)$ is linear, it will require more basis functions to represent $J_t(x)$ than $J_{t+1}(x)$. The reason is that if $\phi^i$ is an $i$ degree polynomial, then $\phi^i(f(x))$ will be a $ki$ degree polynomial if $f(\cdot)$ is a $k$ degree polynomial.
Thus, in general, we will need an expanding basis to represent the functions $J_t(\cdot)$ as we sweep back in time from $T$ to $0$. 
Then, we can characterize the ``true" solution to the Policy Evaluation equation \eqref{pol_eval} as follows.\\

\begin{proposition}\label{PE-true}
The true solution to the policy evaluation equation \eqref{pol_eval} is given by:
$\bar{\alpha}_t^* = \bar{c}_t + \bar{\alpha}_{t+1}^* B_t$, where $B_t = \begin{bmatrix} \beta^{11}, \cdots, \beta^{1N_t}\\ \ddots \\ \beta^{N_{t+1} 1}, \cdots, \beta^{N_{t+1}N_t} \end{bmatrix}$, where $N_{t+1}$ is the number of basis functions required to represent $J_{t+1}(\cdot)$ and $N_t$ is the number of basis functions required to represent $J_t(\cdot)$. 
\end{proposition}
\noindent 
\emph{Proof:} The proof follows by substituting for function $\phi(f(x))$ using \textit{Assumption 1} to equation \eqref{pol_eval} and then solving for the coefficients $\bar{\alpha}^*_t$. (Detailed proof: See Appendix).

\vspace{2mm}
\begin{remark}
The solution above is the discrete time analog of the classical Galerkin procedure for solving Partial Differential Equations (PDEs) \cite{Courant-Hilbert}. This should not come as a surprise since the Policy Evaluation equation \eqref{pol_eval} is really a discrete time analog of the continuous time PDE: $\frac{\partial J}{\partial t} + \bar{c} + \bar{f}\frac{\partial J}{\partial x} = 0$, solved with the terminal condition $J_T(x) = g(x)$, and where $\dot{x} = \bar{f}(x)$ represents the continuous dynamics and the continuous time cost is given by the integral $\int_0^T \bar{c}(x) dt$.
\end{remark}

Next, we show that the RL least squares solution \eqref{RL-LS} converges to the above true solution in the mean square sense as the number of samples $R$ becomes large.\\

\begin{proposition}\label{PE-conv}
\textit{\textbf{PE convergence.}} Let Assumption \ref{basis} hold. Further, let the number of basis functions at time $t$ required be $N_t$. Given that all the required basis functions at time $t$ are considered, the RL least square estimate \eqref{RL-LS} converges to the true solution in the mean square sense. 

\noindent \emph{Proof:} See Appendix.
\end{proposition}

The above result shows that the RL least squares procedure is a randomized approximation of the PE equation. However, the above result follows under an idealized situation when all necessary basis functions are considered, and $R$ becomes very large. Thus, in the following, we characterize the bias and the variance of the estimate, which in turn will allow us to find the sample complexity of the estimates.\\

\begin{corollary} \textit{Bias in RL estimate \eqref{RL-LS}.}
Suppose that the number of basis functions required at time $t$ is $N_t$ and only $N < N_t$ are used. Then, the RL least squares estimate \eqref{RL-LS} is biased for all $\tau < t$. 

\emph{Proof:} See Appendix.
\end{corollary}

Next, we consider the variance of the estimate. The key role here is played by the Gram matrix $\mathcal{G}_t = [<\phi_i, \phi_j>_t]$, $i,j = 1 \cdots N_t$, where $<.,.>_t$ denotes the inner product with respect to the sampling distribution $p_t(.)$. In general, the variance of the solution is determined by the variance in the R-sample empirical Gram matrix estimate $\mathcal{G}^R_t = \frac{1}{R} \Phi_t^R \Phi_t^{R'}$. In the following, we characterize the variance of the empirical Gram matrix $\mathcal{G}^R_t$ for suitable choice of basis functions and sampling distribution.\\

\begin{assumption}
Let the basis functions used at time $t$ be $\{\phi^1, \cdots \phi^{N_t}\}$. We assume that there exists a constant matrix $H_t$ such that $\begin{bmatrix}
\phi^1\\\vdots \\ \phi^{N_t}
\end{bmatrix} = H_t \begin{bmatrix}
1 \\ \vdots \\ x^{M_t}\end{bmatrix},$ i.e., the basis functions at time $t$ can be represented as $M_{t}$ degree polynomials.
\end{assumption}

Let us define 
\begin{equation}\label{Htrans}
    H_t = [H_1, \cdots H_{M_t}] = \begin{bmatrix}
H^1 \\ \vdots \\ H^{N_t} \end{bmatrix},
\end{equation}
i.e., we define the rows and columns of the matrix $H_t$.
The covariance of the error in the LS estimate is given by $P_t^{R} = \sigma_v^2 (G^{R}_t)^{-1} $, where $G^R_t = R \mg^R_t$, and $\mg^R_t = \frac{1}{R} \Phi^R_t \Phi^{R'}_t$ (see Proof of Proposition \ref{PE-conv}). Now, we can characterize the size of the error in the LS estimate as a function of number of samples $R$ required.\\
\begin{theorem}\label{Gram_var}
\textbf{Variance of RL least squares estimate.} Let $x_t \sim \mathcal{N}(0, \sigma_X^2),$ i.e. the sampling distribution is zero mean Gaussian with variance $\sigma_X^2$. Let $\beta <1, \delta>0$, and $n< \infty$ be given. Then, to probabilistically bound the norm of the error covariance:
\begin{equation}
Prob(||P_t^R|| \leq \delta) > 1 - 2e^{-n^2/2},
\end{equation}
the number of samples required are: 
\begin{equation}
    R > \max[ (\frac{n}{\beta}CC')^2 \sigma_{2M_t}^2, \frac{\sigma_v^2 C}{ \delta(1-\beta)}], 
\end{equation}
where 
\begin{equation}
\sigma_{2M_t}^2 = [(4M_t-1)!! - (2M_t-1)!!^2] \sigma_X^{4M_t}
\end{equation}
and $C$ and $C'$ are constants such that $||\mg_t^{-1}|| \leq C$, and $||H_{M_t}||||H^{N_t}||\leq C'$.
\end{theorem}
\noindent \emph{Proof:} See Appendix.


The above result establishes the number of samples required to get an accurate RL least squares estimate.
Next, we shall see the implications of the above result for particular choices of basis functions.

\subsection{Sample Complexity}
\paragraph{Monomial Basis}
For a monomial basis, $C' = 1$ since $H$ is the identity matrix and $N_t = M_t$. Thus, the number of samples has to satisfy:
\begin{equation}\label{C_monomial}
    R\sim O([(4N_t-1)!! - (2N_t-1)!!^2] \sigma_X^{4N_t}),
\end{equation}
for the LS error to be small enough.

\paragraph{Hermite (Orthonormal) Basis} It is well known that the Hermite polynomials form the orthonormal basis for Gaussian sampling distributions \cite{Courant-Hilbert}. Noting that $M_t = N_t$, in the case of Hermite polynomials, owing to their orthonormality, one can show that $C' = \frac{1}{N_t!^2 \sigma_X^{4N_t}}$, and this, in turn, implies that for the Hermite basis, the number of samples $R$ need to satisfy:
\begin{equation}\label{C_Hermite}
    R \sim O\left(\frac{(4N_t-1)!! - (2N_t-1)!!^2}{N_t!^2}\right),
\end{equation}
for the LS error to be small enough.

\paragraph{Nonlinear Basis} In this case, we mean that the cost function $J_t(x) = h(x, \theta_t)$, where $h(x, \theta)$ is a suitable nonlinear approximation architecture, such as a (deep) neural net, parametrized nonlinearly by the (vector) parameter $\theta$. The PE equation in this case becomes: 
\begin{equation}
    h(x, \theta_t) = c(x) + h(f(x), \theta_{t+1}),
\end{equation}
where $\theta_t$ parametrizes the cost function at time $t$, and the same holds for $\theta_{t+1}$. It is reasonable to assume, given the change in the parameter $\theta$ between consecutive steps is small enough, that:
\begin{equation}
    h(x,\theta_t) \approx h(x, \theta_{t+1}) + \sum_{i=1}^{N_t} H_{t+1}^i(x) \delta \theta_i,
\end{equation}
where $H_{t+1}^i(x) = \frac{\partial h(x, \theta)}{\partial \theta_i}|_{\theta_{t+1}}$, where $\theta_i$ are the components of the vector parameter $\theta$. Rewriting the policy evaluation equation, one obtains:
\begin{multline}
    H_{t+1}^1(x) \delta \theta_1 + \cdots + H_{t+1}^{N_t} (x) \delta \theta_{N_t} \\
    = c(x) + \underbrace{[h(f(x), \theta_{t+1}] - h(x, \theta_{t+1})]}_{\delta h_{t+1}(x)} + v,
\end{multline}
where $v$ is a noise term,
which may be written in matrix form for $R$ samples as:
\begin{multline*}
    [\delta \theta_1 \cdots \delta \theta_{N_t}] \underbrace{\begin{bmatrix}
    H_{t+1}^1(x^{(1)}_t) & \cdots & H_{t+1}^1(x^{(R)_t})\\
    \vdots & \vdots & \vdots\\
    H_{t+1}^{N_t}(x^{(1)}_t) & \cdots & H_{t+1}^{N_t}(x^{(R)}_t)
    \end{bmatrix}}_{\mathcal{H}^R_t} \\
    = \underbrace{\begin{bmatrix}
    c(x^{(1)}) ~ \cdots ~ c(x^{(R)})
    \end{bmatrix}}_{\mathcal{C}^R_t} + \underbrace{\begin{bmatrix}
    \delta h_{t+1}(x^{(1)}) ~ \cdots ~
    \delta h_{t+1}(x^{(R)})
    \end{bmatrix}}_{\delta \mathcal{H}^R_t} \\+ \underbrace{\begin{bmatrix}
    v^{(1)}_t & \cdots &
    v^{(R)}_t
    \end{bmatrix}}_{V^R_t},
\end{multline*}
where we assume that the noise terms $v^{(i)}_t$ are i.i.d with variance $\sigma_v^2$  as before.
The least squares solution to the above equation is, as usual, given by:
\begin{equation}\label{NL-RL}
    \delta \theta^R_t = \mathcal{C}^R_t \mathcal{H}^{R'}_t (\mathcal{H}^R_t\mathcal{H}^{R'}_t)^{-1} + \delta H^R_t \mathcal{H}^{R'}_t(\mathcal{H}^R_t\mathcal{H}^{R'}_t)^{-1}.
\end{equation}
Denote the empirical Gram matrix for the instantaneous basis functions, $H_{t+1}^i(x)$, at time $t+1$ as:
\begin{equation}
    \mg_{t}^R = \frac{1}{R}\mathcal{H}^R_t \mathcal{H}^{R'},
\end{equation}
and thus, the covariance of the error in the LS solution is given by $\frac{\sigma_v^2}{R} (\mg_t^R)^{-1}$. Note that this situation is no different from the linear case considered previously, in that we are approximating the change in the cost function via the change in the parameter $\theta$ at time $t+1$, and the instantaneous basis functions $H_{t+1}^i(x)$. The primary difference with the linear case is that our basis functions $H_{t+1}^i(x)$ change with time unlike the fixed basis in the linear case. Suppose that we have:
$
    \begin{bmatrix}
    H_{t+1}^1(x)\\
    \vdots\\
    H_{t+1}^{N_t}(x)
    \end{bmatrix}
    = \mathcal{D}_{t+1}\begin{bmatrix}
    1\\x\\ \vdots \\ x^{M_t}
    \end{bmatrix},
$
then we are back to the conditions enumerated in Theorem \ref{Gram_var}. Therefore, it follows that the number of samples $R$ required such that we get a small enough error in the solution should be:
\begin{equation}\label{C_nonlinear}
    R\sim O([(4M_t-1)!! - (2M_t-1)!!^2] \sigma_X^{4M_t}).
\end{equation}

\textbf{Discussion.} 
It can be seen clearly from above that choosing orthonormal (o.n.) polynomials to the sampling distribution greatly reduces the variance of the RL least squares estimate \eqref{RL-LS}. Further, the variance of the estimate if we use unnormalized polynomials is very high. The situation is no different even when using a nonlinear basis. Typically in RL, one uses rollouts and the basis functions are almost never chosen to satisfy orthonormality. Further, in general, since the rollout sampling distributions $p_t(\cdot)$ need not be Gaussian, finding such o.n. basis functions is a challenge in itself since one does not have an idea about these sampling distributions in advance, or they are never known explicitly. However, in our opinion, the sampling distributions can, and should, be chosen at our convenience and need not arise from rollouts if viewed in the context of the Galerkin interpretation of Proposition \ref{PE-true}. In fact, the above analysis suggests that if we are to reduce the variance of the estimates, the sampling distributions should not be chosen from rollouts.

A significantly more intractable problem arises due to the exponential growth of the basis functions required at every time step. Suppose that the dynamics could be represented by a $k$ degree polynomial, and the terminal cost function was degree $N_T$, then the number of basis functions required at time $t$ is $k^{T-t}N_T$. When coupled with the variance estimate above, one can see that, even with o.n. basis, this leads to an explosive variance growth: a factorial-exponential rate of growth.
\section{A Perturbation Structure for PE}
In this section, we present a construct that ideally allows a perturbation structure to the Policy Evaluation equations, i.e., a structure where the higher order terms do not affect the evolution of the lower order terms which implies that we can close our computations at any desired order without incurring an error.

\begin{figure}[h!]
    \centering
    \includegraphics[width= 8 cm,height= 9 cm,keepaspectratio]{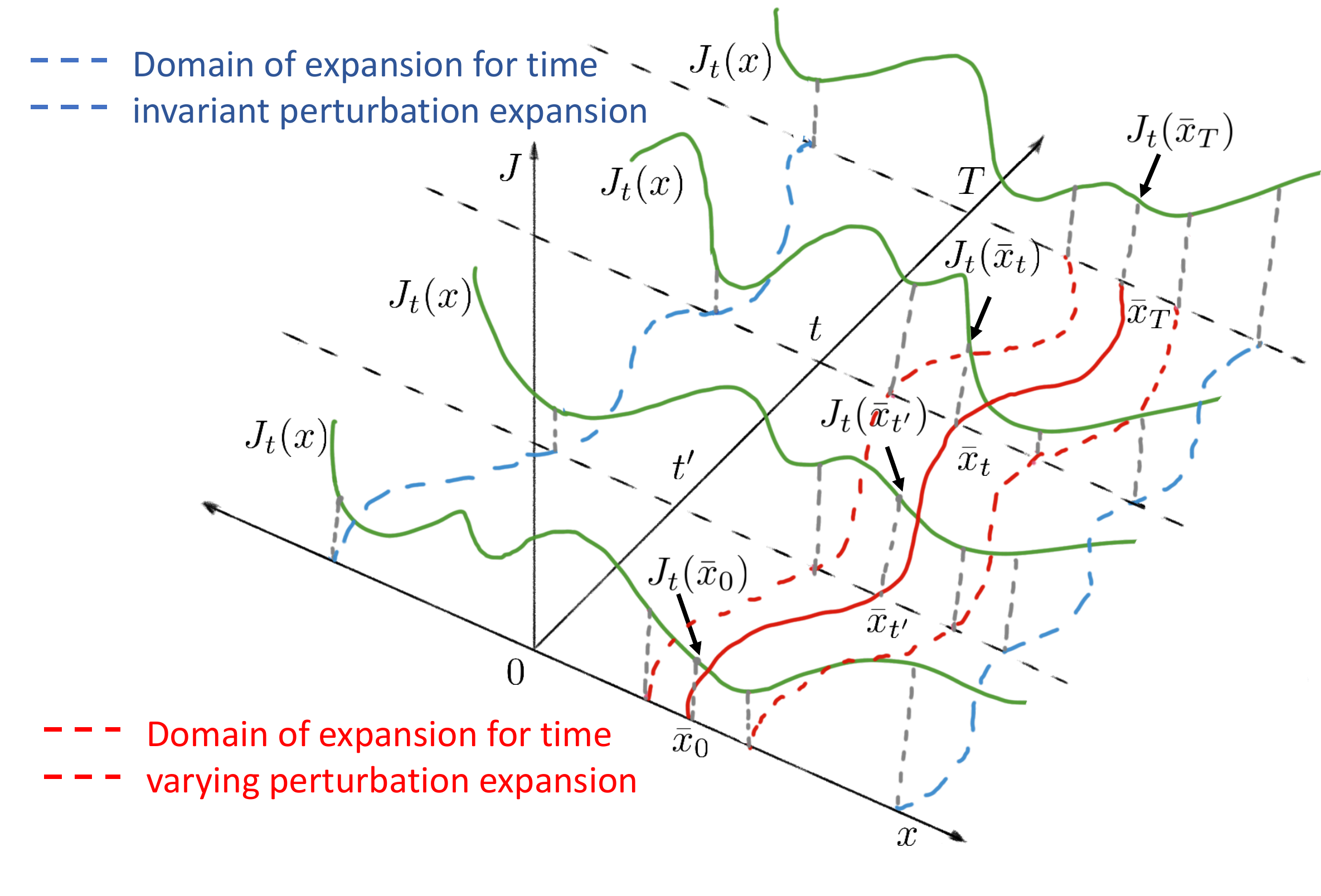}   
    \vspace{-3mm}
    \caption{Time-varying versus time-invariant perturbation expansion.}
    \label{N_terms}
\end{figure}
Recall the Policy evaluation equation \eqref{pol_eval}. Now, consider that we are given a nominal trajectory under the dynamics, say $\bar{x}_t = f(\bar{x}_{t-1})$, $t=0, 1,\cdots T$ given some initial condition $\bar{x}_0 = x_0$. Next, expand the dynamics about this nominal trajectory as: 
\begin{equation}
f(x_t) = f(\bar{x}_t + \delta x_t)= f(\bar{x}_t) + F^1_t \delta x_t + F^2_t\delta x_t^2 +\cdots, 
\end{equation}
where $F_t^i$ denotes the $i^{th}$ term in the expansion of the dynamics around the nominal trajectory. Similarly:
\begin{multline}
    J_{t+1}(f(x_t)) = J_{t+1}(f(\bar{x}_t + \delta x_t)) \\ = J_{t+1}(f(\bar{x}_t)) + K_{t+1}^1 (F_t^1\delta x_t + \frac{1}{2} F_t^2 \delta x_t^2+ \cdots) \\ + \frac{1}{2}K_{t+1}^2 (F_t^1\delta x_t + F_t^2 \delta x_t^2+ \cdots)^2+ \cdots,
\end{multline}
where $K_{t+1}^i$ denotes the $i^{th}$ term in the Taylor expansion of $J_{t+1}(f(x_t))$ around the nominal. Similarly the incremental cost function $c_t(x_t) = c_t(\bar{x}_t) + C_t^1 \delta x_t + \frac{1}{2}C_t^2 \delta x_t^2 + \cdots$, and the optimal cost function at time $t$, $J_t(x_t) = J_t(\bar{x}_t) + K_t^1 \delta x_t + \frac{1}{2} K_t^2 \delta x_t^2+ \cdots$. Substituting the above expressions into the policy evaluation equation \eqref{pol_eval}, we obtain: 
\begin{multline}
J_t(x_t)  = \bar{J}_t + K_t^1 \delta x_t + \frac{1}{2}K_t^2 \delta x_t^2 +\cdots \\
  = (\bar{c}_t + C_t^1 \delta x_t + \frac{1}{2} C_t^2 \delta x_t^2+ \cdots)+ \bar{J}_{t+1} \\ + K_{t+1}^1 (F_t^1\delta x_t + \frac{1}{2} F_t^2 \delta x_t^2+ \cdots) \\+ 
  \frac{1}{2}K_{t+1}^2 (F_t^1\delta x_t + F_t^2 \delta x_t^2+ \cdots)^2 + \cdots,
\end{multline}
where $\bar{J}_t = J(\bar{x}_t)$, $c(\bar{x}_t) = \bar{c}_t$ and $\bar{J}_{t+1} = J_{t+1}(\bar{x}_{t+1}) = J_{t+1}(f(\bar{x}_t))$. Now, grouping the different terms of $\delta x_t^i$ on both sides of the equation allows for writing the following vector-matrix form equation as:
\begin{multline}\label{PPE}
    [K_t^1 \, K_t^2  \cdots] = [C_t^1 \,C_t^2 \cdots]\\
    + [K_{t+1}^1 \,K_{t+1}^2,\cdots]\underbrace{\begin{bmatrix}
    F_t^1 & F_t^2 & \cdots \\
    0 & (F_t^1)^2 &  \cdots \\
    0 & 0 & \cdots\\
    \vdots & \vdots & \vdots \\
    0 & 0 & \cdots
    \end{bmatrix}}_{\mathcal{B}_t},
\end{multline}
note the equations have a beautiful perturbation/ upper triangular structure. 
Practically, this means that we can close our computations at any order we desire without worrying about the effect of the higher order terms on the lower order terms, given we have knowledge of the dynamics, and hence, the Taylor coefficients $F_t^i$. 

A special case of the above equation is when $f(0) = 0$, and the nominal trajectory is simply $\bar{x}_t = 0$. In such a case the expansion is about a nominal trajectory that stays at the origin. In such a case, the Taylor coefficients, rather than being time varying, will be time invariant, i.e., $F^1, F^2,\cdots$ etc. Typically, we are given problems where the initial state $x_0 \neq 0$, and can be far from the origin. In such a case, we may see that the number of terms required for a static expansion, i.e., about $\bar{x}_t = 0$, will require far more terms than would an expansion that was centered on a nominal trajectory starting at $x_0$. The situation is illustrated in Fig. \ref{N_terms}. Thus, it is much more efficient to seek the time varying expansion above. In particular, we shall explore the implication further when we solve the perturbed policy evaluation (PPE) equation \eqref{PPE}.\\

\begin{remark}
Note that given the Taylor coefficients $F_t^i$, the second and higher rows of the ``dynamics" matrix encoding the structure of the PPE equation are perfectly known. This knowledge can be used to solve the PPE equation in a ``model based" fashion, as opposed to a model-free approach, where this structure is actually teased out of the data from the system.  
\end{remark}

\subsection{RL type solution to the PPE}
First, we show a model-based solution to the PPE equation \eqref{PPE}, i.e., one where we explicitly estimate the first $M$ Taylor coefficients of the dynamics $\mf_t \equiv [F_t^1,F_t^2,\cdots F_t^M]$ which are then substituted into \eqref{PPE} to solve the PE equation. Next, we show how this can be extended to the model-free case: one where we do not solve for $\mf_t$, and instead directly solve \eqref{PPE} and infer the matrix $\mathcal{B}_t$ from the system data.

We can write the following approximation, after neglecting the higher order terms beyond $\delta x_t^M$:\\ $\delta x_{t+1}^{(i)} \approx \mf_t \begin{bmatrix}
\delta x_t^{(i)}\\
(\delta x_t^{(i)})^2\\
\vdots\\
(\delta x_t^{(i)})^M
\end{bmatrix} + v_t^{(i)}$, where as before $v_t^{(i)}$ is an i.i.d. noise sequence and $i=1,2\cdots R$. A least squares estimate of $\mathcal{F}_t$ is quite straightforward and may be written as:
\begin{equation}\label{MB-LS}
    \mf_t^R = \delta X_{t+1}^R \delta \chi_t^{R'}(\delta \chi_t^{R} \delta \chi_t^{R'})^{-1},
\end{equation}
where $\delta \chi_t^R = \begin{bmatrix}
\delta x_t^{(1)} & \cdots & \delta x_t^{(R)}\\
\vdots & \vdots &\vdots\\
 (\delta x_t^{(1)})^M & \cdots & (\delta x_t^{(R)})^M
\end{bmatrix}$, and $\delta X_{t+1}^R = \begin{bmatrix}
\delta x_{t+1}^{(1)} & \cdots & \delta x_{t+1}^{(R)}
\end{bmatrix}$, where $\delta x_{t}^{(i)} = x_t^{(i)} - \bar{x}_t $, and $\delta x_{t+1}^{(i)} = f(x_t^{(i)}) - \bar{x}_{t+1}$, where $\bar{x}_{t+1} = f(\bar{x}_t)$. Further, we assume that $\delta x_t^{(i)} \sim \mathcal{N}(0,\sigma_X^2)$. Thus, the data is obtained by perturbing the system from the nominal trajectory.

The following development characterizes the error in the LS solution \eqref{MB-LS} incurred from neglecting the higher order terms of the dynamics (beyond $\delta x_t^M$) and shows that it can be made arbitrarily small by choosing the perturbation $\delta x_t$ to be suitably small. In the following, unlike in Section III, the Gram matrix $\mg$ size will not change, since we are looking at order $M$ approximation throughout time. \\

\begin{lemma}
Let $\Delta_t^R = [\Delta_t^{R,1}, \Delta_t^{R,2} \cdots \Delta_t^{R,M}]$, where $\Delta_t^{R,l} = \frac{1}{R}\sum_{k>M} \sum_{i=1}^R F_t^k (\delta x_t^{(i)})^k(\delta x_t^{(i)})^l$. Let the empirical Gram matrix $\mg^R = [\mg^R_{ij}]$, where $i,j=1,2,\cdots M$, and $\mg^R_{ij} = \frac{1}{R}\sum_{k=1}^R (\delta x_t^{(k)})^i (\delta x_t^{(k)})^j $. Then, 
\begin{equation}
    \mf_t^R = \mf_t  + \Delta_t^R {(\mg^R)^{-1}} + \frac{1}{R}V_t^R \delta \chi_t^{R'}(\mg^R)^{-1}.
\end{equation}
\end{lemma}
\noindent \emph{Proof:} See Appendix. 


The above result makes it clear that our estimates are biased for any finite $R$ (in fact, even for the limit), but if $\Delta^R_t$ is small enough, then this bias can be made small. In the following, we show precisely such a result. \\

\begin{proposition}\label{LS-e}
Let $\Delta_t = \lim_R \Delta_t^R$ and let $\mg = \lim_R \mg^R$. Given any $M$, any $\epsilon > 0$, there exists a variance $\sigma_X^2 < \infty$ such that $|\Delta_t^j| \leq \epsilon |\tilde{f}_t^j|$, where $\mf_t\mg = [\tilde{f}_t^1, \cdots, \tilde{f}_t^j, \cdots \tilde{f}_t^M]$. 
\end{proposition}
\noindent \emph{Proof:} See Appendix. \\


Next, we have the following consequence.
\begin{corollary}
The least squares estimate \eqref{MB-LS}, $\mf_t^R \rightarrow \mf_t + \Delta_t \mg^{-1}$ as $R \rightarrow \infty$ in mean square sense. 
\end{corollary}
\noindent \emph{Proof:} See Appendix. \\

\textbf{\textit{Accuracy of the Solution.}} To understand the result above, let us rewrite the limiting solution as $\mf_t^R = (\mf_t\mg + \Delta_t) \mg^{-1}$. Thus, the limiting solution can be understood as the solution to the linear equation $\mf_t\mg = \tilde{\mf}_t +  \Delta_t$ where $\tilde{\mf}_t = \mf_t\mg$. We have shown in Proposition \ref{LS-e} that for small enough variance, $||\Delta_t|| \leq \epsilon ||\tilde{\mf}_t||$, and thus we can expect a small error on the right hand side of the equation above. However, albeit the error $\Delta_t$ may be small compared to the ``signal $\tilde{\mf}_t$", the actual error in the solution is affected by the conditioning of the Gram matrix $\mg$. In fact, one can show that:
\begin{equation}
\displaystyle{\frac{||\mf_t^{\infty} - \mf_t||}{||\mf_t||} \leq \kappa(\mg) \epsilon},
\end{equation}
where $\mf_t^{\infty} = \lim_R \mf_t^R$, and $\kappa(\mg)$ denotes the condition number of the Gram matrix $\mg$. Thus, the true pacing item in the accuracy of the solution is the conditioning of the matrix $\mg$. In fact, if the input is Gaussian, then the conditioning of the matrix rapidly deteriorates as $M$ increases since the higher moments of a Gaussian increase as $(2M-1)!! \sigma_X^{2M}$, i.e, in factorial-exponential fashion. Thus, in practice, one cannot make $M$ large as the solution becomes highly sensitive due to the ill conditioning of the Gram matrix.

\textbf{\textit{Variance of Solution.}} As shown previously, the variance of the solution is directly proportional to the variance in the empirical Gram matrix $\mg^R$, and thus, the number of samples required is $R\sim \mathcal{O} ( [(4M-1)!! - (2M-1)!!^2]\sigma_X^{4M})$ as in \eqref{C_monomial}.
Hence, the central role in the accuracy and in the variance of the least squares estimate is played by the empirical Gram matrix $\mg^R$.\\


\textbf{Model-based or Model-free?}
The number of computations in the model-free method is $\approx 2 \times$ the computation in the model-based method, since $R$ is typically large. Thus, the model-free approach will result in higher computational efforts or $\approx 2 \times$ the variance as compared to model-based methods. \emph{(Detailed Discussion: See Appendix. 
)}


\textit{\textbf{Orthonormal and Nonlinear Basis Functions:}} As we saw previously, the variance of the RL-LS solution \eqref{RL-LS} is highly significantly reduced when using o.n. basis functions such as the Hermite polynomials, see \eqref{C_monomial} vs \eqref{C_Hermite}. Thus, it is of interest to see if any added advantage can be gained by using such o.n. functions in the PPE case.\\
We may write the Hermite polynomials in terms of the monomials as: $\begin{bmatrix} h_1(\delta x)\\ \vdots \\ h_N(\delta x)\end{bmatrix} = H \begin{bmatrix}\delta x \\ \vdots \\ \delta x^N
\end{bmatrix}$,
where $h_i(\cdot)$ represent the Hermite polynomials, and $H$ is a lower triangular matrix that encodes the linear transformation from the monomials to the Hermite polynomials. Similarly, let $H^{-1}$ represent the inverse transformation from the Hermite polynomials to the monomials. Then, the PPE equation \ref{PPE} may be written in the Hermite basis as:
\begin{equation}
    \mk_t^H = \mathcal{C}^H_t + \mk_{t+1}^H \mathcal{B}_t^H,
\end{equation}
where $\mk_t^H$, and $\mathcal{C}_t^H$ consist of the Hermite coefficients (rather than the Taylor coefficients) and $\mathcal{B}_t^H = H \mf_t H^{-1}$. Note that $H$ and $H^{-1}$ are lower triangular, while $\mf_t$ is upper triangular, and therefore, $\mathcal{B}_t^H$ is fully populated, i.e., the PPE equations in the Hermite basis lose their perturbation/ upper triangular structure. This is the primary shortcoming of the o.n. representation, in particular, we shall show the necessity of the perturbation structure to an accurate solution in the following section after we outline the stochastic case.

Further, in the nonlinear basis case, nothing changes from the LS problem in \eqref{NL-RL}, since the formulation cannot distinguish deviations from a trajectory. Thus, the equations for the parameters, in general, will be fully coupled, and not have a perturbation structure.
\section{The Stochastic Case}
Now, we consider the stochastic case. The policy evaluation equation in the stochastic case becomes:
\begin{equation}
    J_t(x_t) = c_t(x_t) + E[J_{t+1}(x_{t+1})],
\end{equation}
where $x_{t+1} = f(x_t) + \omega_t$, and $\omega_t$ is a Gaussian white noise sequence. \\
Let us define the nominal trajectory as the noise free case, i.e., $\bar{x}_{t+1} = f(\bar{x}_t)$, given some nominal initial condition $\bar{x_0}$. Next, let us consider perturbations $\delta x_t$ about this nominal. Thus, we get 

$J_{t+1}(x_{t+1}) = J_{t+1}(f(x_t) + \omega_t) = J_{t+1}(f(\bar{x}_t + \delta x_t) + \omega_t)$. Using: 

$f(\bar{x}_t + \delta x_t) = f(\bar{x}_t) + F_t^1 \delta x_t + F_t^2 \delta x_t^2 + \cdots$, 

$ J_t(x_t) = J_t(\bar{x}_t) + K_t^1 \delta x_t + K_t^2 \delta x_t^2 + \cdots$, 

\noindent and thus, 

$J_{t+1}(x_{t+1}) = J_{t+1}(\bar{x}_{t+1}) + K_{t+1}^1 (\delta f_t + \omega_t) + K_{t+1}^2 (\delta f_t + \omega_t)^2 + \cdots$,

\noindent where $\delta f_t = F_t^1 \delta x_t + F_t^2 \delta x_t^2 + \cdots$. Thus, using the fact that the noise sequence is white and Gaussian:
\begin{multline}
    E[J_{t+1}(x_{t+1})] = (\bar{J}_{t+1} + K_{t+1}^2 + \cdots) +\\ (K_{t+1}^1 F_t^1 + K_{t+1}^3 (3F_t^1) + \cdots)\delta x_t + \cdots.
\end{multline}
In particular, we may write:
\begin{equation}
    E[J_{t+1}(x_{t+1})] = \sum_k G^k(\mf_t, \mk_{t+1}) \delta x_t^k,
\end{equation}
where recall that $\mf_t$ and $\mk_{t+1}$ are the Taylor coefficients of the dynamics at time $t$  and the cost function at time $t+1$ respectively, and $G^k(\mf_t,\mk_{t+1})$ are suitably defined functions of the Taylor coefficients $\mf_t, \mk_{t+1}$. Therefore, the policy evaluation equations in the stochastic case become (equating the coefficients of the different powers of $\delta x_t$ on both sides of the equation):
$    \bar{J}_t = \bar{c}_t +        G^0(\mf_t,\mk_{t+1}), ~
    K_t^1 = C_t^1 + G^1(\mf_t, \mk_{t+1}),~
    K_t^2 = C_t^2 + G^2(\mf_t, \mk_{t+1}),~
    \cdots. $
The import of the above equations is that, unlike in the deterministic case, there is no perturbation structure in the stochastic case, and as a result all terms affect all other terms. This is the fundamental issue with the stochastic case, and makes the computation of a solution, in general, intractable. 

\subsection{The Necessity of the Perturbation Structure.}
Now, let us consider why the perturbation structure is critical to a solution. Write the $G^k(\cdot)$ above as $G^k (\mf_t, \mk_{t+1}) = G^{k,M}(\mf_t^M, \mk_{t+1}^M) + \delta G^{k,M}$, i.e., we keep only the first $M$ terms in the Taylor series $\mf_t$ and $\mk_{t+1}$, denoted by $\mf_t^M$ and $\mk_{t+1}^M$ respectively, and $\delta G^{k,M}$ denotes the error resulting from neglecting the higher order terms. Then, for an accurate solution, we require that $M$ is large enough such that $||\delta G^{k,M}|| \leq \epsilon ||G^{k,M}||$, for suitably small $\epsilon$. In the case of the PPE eq. (\ref{PPE}), we can choose any $M$, no matter how small, and still be assured of zero error in computing those coefficients, given the model knowledge. However, even with model knowledge, the computations are intractable in the stochastic case since the number of terms $M$ can be expected to be large for most problems (except linear problems).

Next, consider the RL scenario: the stochastic case is essentially similar to the deterministic case in that the LS solution is still given by a slight modification of eq. \eqref{RL-LS} where now:
$    \Phi^R_{t+1} = \begin{bmatrix}
    \phi^1(x_{t+1}^{(1)}) & \cdots & \phi^1(x_{t+1}^{(R)})\\
    \vdots & \vdots & \vdots\\
    \phi^N(x_{t+1}^{(1)}) & \cdots & \phi^N(x_{t+1}^{(R)}),
    \end{bmatrix} $
and the random next states $x_{t+1}^{(i)}$ are sampled from the stochastic dynamics: $x_{t+1}^{(i)} = f(x_t^{(i)}) + \omega_t^{(i)}$. Thus, there will be a higher variance in the elements of the $\Phi_{t+1}^R$ matrix owing to the input noise $\omega_t$. In this case, due to the additional variance, $\Phi_{t+1}^R \Phi_t^{R'}$ will take longer to converge than $\Phi_t^R \Phi_t^{R'}$. Nonetheless, the error covariance estimates do not change from the deterministic case, and thus, the sample complexity depends on the empirical Gram matrix $\mg^R$, and remains the same (at least in the $\mathcal{O}(\cdot)$ sense).

\textit{Bias and Variance.} It is inevitable, even with the perturbation structure of the deterministic case, that the higher order terms corrupt the computations of the lower order terms in policy evaluation when done via RL. However, as shown in Proposition \ref{LS-e}, this error can be made arbitrarily small by controlling the input variance, and we can close our computations at any suitable $M$ such that the Gram matrix is well conditioned. However, in the stochastic case, we do not have this luxury since for a good solution, we will have to choose a suitably large $M$. This, in turn, implies that the bias is adversely affected due to the ill-conditioning of the Gram matrix, $\mg$, for large $M$. Furthermore, owing to the same reason, the variance of the empirical Gram matrix $\mg^R$ is also bound to be very high. Thus, in the stochastic case, we see that the RL solution will be overwhelmed by both the bias as well as the variance, thereby making the solution highly unreliable, i.e., one can always obtain a solution to the stochastic case by solving a least squares problem, but it is bound to be inaccurate owing to the fundamental structure of the problem.
Finally, the situations does not change if we use an o.n. basis like the Hermite basis or a nonlinear basis, due to the lack of a perturbation structure.

\section{Empirical Results}
\label{section:Results}
Thus far in this paper, we have established theoretical results that show that finding a global (higher order) solution for the policy evaluation problem is subject to very high error and variance. In this section, we provide empirical evidence of this explosive growth of errors. In order to accomplish this, we need a system for which we know the true solution to the PE problem. To this end, we assume a discretized system of the form:
\begin{align}
x_{t+1} = x_t + \delta(-x_t + \epsilon x_t^3 ), 
\end{align}
where $\delta$ is the time discretization. Let us define the terminal cost to be of the quadratic form: 
$
    J_T(x_T) = \alpha x_T^2
$
and the incremental cost to be: $c_t(x_t) = cx_t^2$, so that the DP equation can be written as:
\begin{align}
J_t(x_t) = cx_t^2 + J_{t+1}(x_{t+1}).
\end{align}
Since this system has a polynomial nonlinearity, it is straightforward to find the cost functions backward in time as polynomials. Given that we have this true answer, we can now find the error in our solution as we sweep back in time, as a function of the number of samples, as well as the exploration parameter $\sigma_X$ which determines whether we explore locally or globally.
We show results for polynomial approximations of order $N = M = \{6,12,18\}$ as we back propagate 3 steps in time. The required number of basis functions at time $t = \{T, T-1, T-2, T-3\}$ are $M_t = N_t = \{2, 6, 18, 54 \}$, respectively. The parameter values assumed are: $\epsilon = 1, \delta = 0.1, c = 10,$ and $\alpha = 10$. 

\begin{figure}[h!]
    \centering
    \includegraphics[width= 8 cm,height= 9 cm,keepaspectratio]{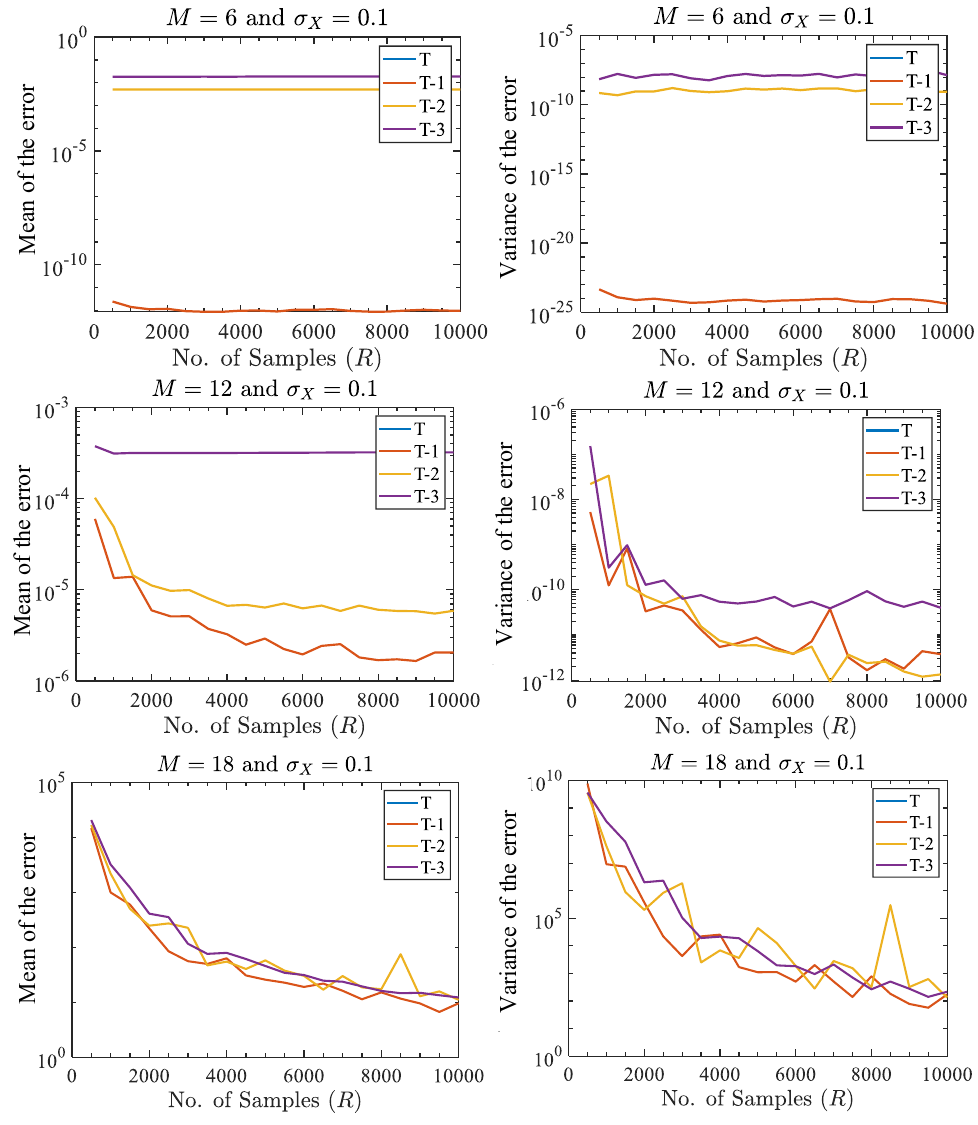} 
    \vspace{-3mm}
    \caption{Results for varying number of samples for exploration parameter $\sigma_X = 0.1$ and different orders of approximation, $M=6,12,18$, for the backward sweep in time till $T-3$.}
    \label{Fig_res1}
\end{figure}
For a ``small" exploration noise of $\sigma_X = 0.1$, as we increase the number of basis functions $M_t$, better results in the mean error are obtained as we propagate back in time due to the increase in number of required basis functions (1st column of Fig.~\ref{Fig_res1}). Although the mean error improves till M= 12, there is a marked increase in the variance of the error with increased basis function (2nd column of Fig.~\ref{Fig_res1}) from 12 to 18.
Also, notice the large number of samples required to reduce the variance of the error even for a one-dimensional problem.

\begin{figure}[h!]
    \centering
    \includegraphics[width= 8 cm,height= 9 cm,keepaspectratio]{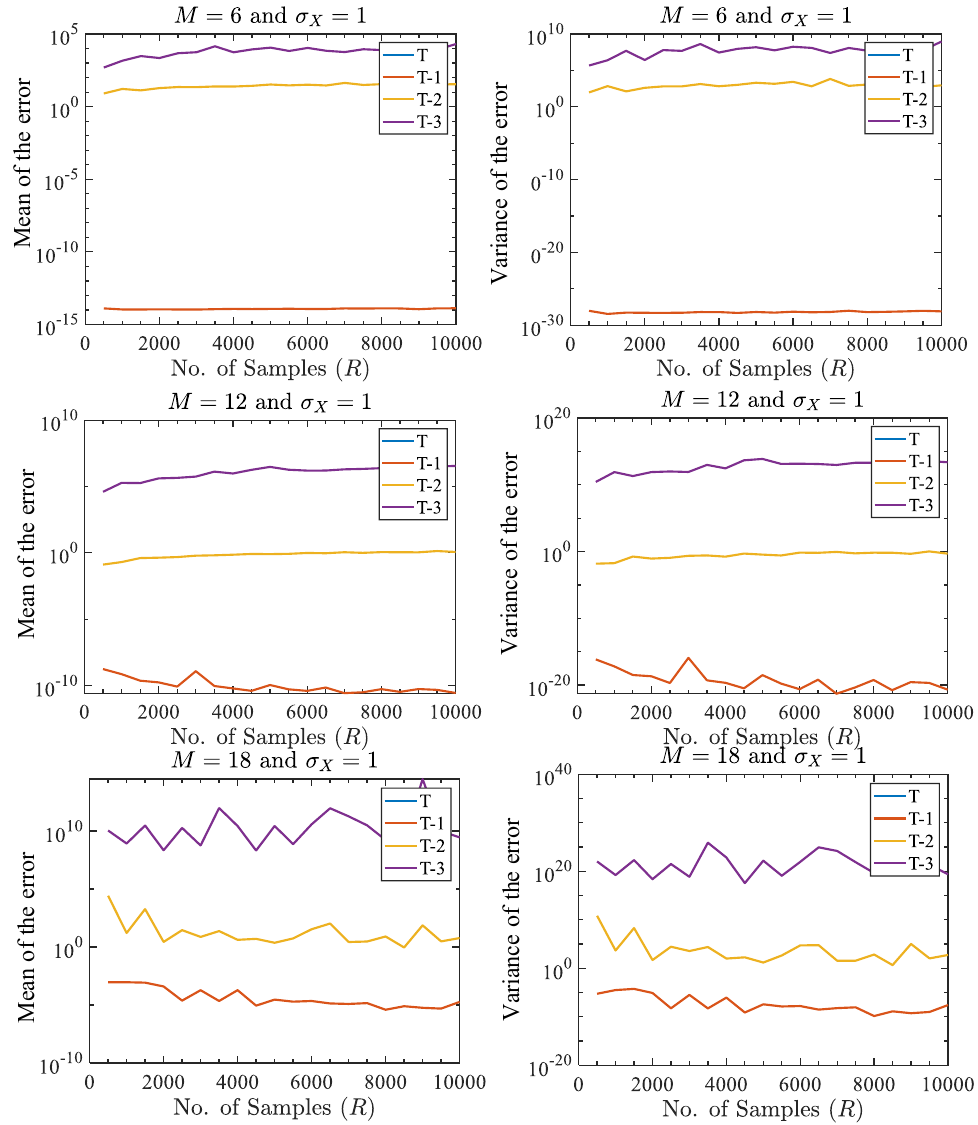} 
    \vspace{-2mm}
    \caption{Results for varying number of samples for $\sigma_X = 1$.}
    \label{Fig_res2}
\end{figure}
For the large exploration case with $\sigma_X =1$, the mean error and its variance become very high (Fig.~\ref{Fig_res2}). Also, notice that even with a large number of samples, the error values do not decrease as in the small noise case of $\sigma_X = 0.1$. Another important observation here is that albeit the performance at $T-1$ is very good, as we sweep back in time, the errors and their variance show explosive growth and become unacceptable. \\
Thus, albeit preliminary, this empirical evidence suggests that solution accuracy, measured via the mean error and its variance, is indeed adversely affected by higher order approximations and/ or large exploration. Conversely, the only way to ensure accuracy is to have a suitably low order approximation which needs to be enforced by a suitably local exploration.

\vspace{-2mm}
\section{Conclusions and Implications}
\vspace{-1mm}
\label{section:conclusion}
In this paper, we have studied the inherent structure of the Reinforcement Learning problem. Concentrating on the policy evaluation problem, we have shown that unless we seek local solutions, the answers found are bound to suffer from high variance, and thus, be inaccurate. In particular, the deterministic problem has a perturbation structure that can be exploited to obtain arbitrary accurate local solutions. It is also shown that the stochastic problem lacks the perturbation structure, and thus, is bound to suffer from high variance when solved in an RL fashion. \\
The primary issue one has to worry about now: ``what now for stochastic control?" It is intractable so the best seems to be the deterministic approximation. However, 
the deterministic solution is optimal locally, and thus, when allied with replanning of the nominal trajectory, we can recover at least near-optimal solution for the stochastic case. We also conjecture that this local replanning based approach is ``fundamentally" the best that one can hope to achieve via computation/ RL. Our future research will concentrate on doing an extensive version of the experiments that we started in this paper, and provide further evidence that local RL methods are the ones to pursue.

\vspace{-2mm}
\bibliographystyle{IEEEtran}
\bibliography{CDC_refs,MAP_refs1,naveed_references,MohammadRafi}

\begin{thebibliography}{10}
\providecommand{\url}[1]{#1}
\csname url@samestyle\endcsname
\providecommand{\newblock}{\relax}
\providecommand{\bibinfo}[2]{#2}
\providecommand{\BIBentrySTDinterwordspacing}{\spaceskip=0pt\relax}
\providecommand{\BIBentryALTinterwordstretchfactor}{4}
\providecommand{\BIBentryALTinterwordspacing}{\spaceskip=\fontdimen2\font plus
\BIBentryALTinterwordstretchfactor\fontdimen3\font minus
  \fontdimen4\font\relax}
\providecommand{\BIBforeignlanguage}[2]{{%
\expandafter\ifx\csname l@#1\endcsname\relax
\typeout{** WARNING: IEEEtran.bst: No hyphenation pattern has been}%
\typeout{** loaded for the language `#1'. Using the pattern for}%
\typeout{** the default language instead.}%
\else
\language=\csname l@#1\endcsname
\fi
#2}}
\providecommand{\BIBdecl}{\relax}
\BIBdecl

\bibitem{bertsekas1}
D.~P. Bertsekas, \emph{Dynamic Programming and Optimal Control, vols I and
  II}.\hskip 1em plus 0.5em minus 0.4em\relax Cambridge, MA: Athena Scientific,
  2012.

\bibitem{bellman}
R.~E. Bellman, \emph{Dynamic Programming}.\hskip 1em plus 0.5em minus
  0.4em\relax Princeton, NJ: Princeton University Press, 1957.

\bibitem{parr3}
M.~Lagoudakis and R.~Parr, ``Least squares policy iteration,'' \emph{Journal of
  Machine Learning Research}, vol.~4, pp. 1107--1149, 2003.

\bibitem{RLHD1}
R.~Akrour, A.~Abdolmaleki, H.~Abdulsamad, and G.~Neumann, ``Model free
  trajectory optimization for reinforcement learning,'' in \emph{Proc. of the
  ICML}, 2016.

\bibitem{RLHD2}
E.~Todorov and Y.~Tassa, ``Iterative local dynamic programming,'' in
  \emph{Proc. of the IEEE Int. Symposium on ADP and RL.}, 2009.

\bibitem{RLHD3}
E.~Theodorou, Y.~Tassa, and E.~Todorov, ``Stochastic differential dynamic
  programming,'' in \emph{Proc. of the ACC}, 2010.

\bibitem{RLHD4}
S.~Levine and P.~Abbeel, ``Learning neural network policies with guided search
  under unknown dynamics,'' in \emph{Advances in NIPS}, 2014.

\bibitem{RLHD5}
S.~Levine and K.~Vladlen, ``Learning complex neural network policies with
  trajectory optimization,'' in \emph{Proc. of the ICML}, 2014.

\bibitem{azar2012sample}
M.~G. Azar, R.~Munos, and B.~Kappen, ``On the sample complexity of
  reinforcement learning with a generative model,'' \emph{arXiv preprint
  arXiv:1206.6461}, 2012.

\bibitem{kakade2003sample}
S.~M. Kakade, ``On the sample complexity of reinforcement learning,'' Ph.D.
  dissertation, UCL (University College London), 2003.

\bibitem{munos2008finite}
R.~Munos and C.~Szepesv{\'a}ri, ``Finite-time bounds for fitted value
  iteration.'' \emph{Journal of Machine Learning Research}, vol.~9, no.~5,
  2008.

\bibitem{Recht2019sample}
S.~Dean, H.~Mania, N.~Matni, B.~Recht, and S.~Tu, ``On the sample complexity of
  the linear quadratic regulator,'' \emph{Foundations of Computational
  Mathematics}, pp. 1--47, 2019.

\bibitem{recht2019tour}
B.~Recht, ``A tour of reinforcement learning: The view from continuous
  control,'' \emph{Annual Review of Control, Robotics, and Autonomous Systems},
  vol.~2, pp. 253--279, 2019.

\bibitem{henderson2018deep}
P.~Henderson, R.~Islam, P.~Bachman, J.~Pineau, D.~Precup, and D.~Meger, ``Deep
  reinforcement learning that matters,'' in \emph{Thirty-Second AAAI Conference
  on Artificial Intelligence}, 2018.

\bibitem{mohamed2020optimality}
M.~N.~G. Mohamed, S.~Chakravorty, and R.~Wang, ``Optimality and tractability in
  stochastic nonlinear control,'' \emph{arXiv preprint arXiv:2004.01041}, 2020.

\bibitem{Mayne_1}
D.~Q. Mayne, ``Model predictive control: Recent developments and future
  promise,'' \emph{Automatica}, vol.~50, pp. 2967--2986, 2014.

\bibitem{Mayne_2}
J.~B. Rawlings and D.~Q. Mayne, \emph{Model Predictive Control: Theory and
  Design}.\hskip 1em plus 0.5em minus 0.4em\relax Madison, WI: Nob Hill, 2015.

\bibitem{T-MPC1}
L.~Chisci, J.~A. Rossiter, and G.~Zappa, ``Systems with persistent
  disturbances: predictive control with restricted contraints,''
  \emph{Automatica}, vol.~37, pp. 1019--1028, 2001.

\bibitem{T-MPC2}
J.~A. Rossiter, B.~Kouvaritakis, and M.~J. Rice, ``A numerically stable state
  space approach to stable predictive control strategies,'' \emph{Automatica},
  vol.~34, pp. 65--73, 1998.

\bibitem{T-MPC3}
D.~Q. Mayne, E.~C. Kerrigan, E.~J. van Wyk, and P.~Falugi, ``Tube based robust
  nonlinear model predictive control,'' \emph{International journal of robust
  and nonlinear control}, vol.~21, pp. 1341--1353, 2011.

\bibitem{Mayne_3}
D.~Mayne, ``Robust and stochastic mpc: Are we going in the right direction?''
  \emph{IFAC-PapersOnLine}, vol.~48, no.~23, pp. 1 -- 8, 2015, 5th IFAC
  Conference on NMPC 2015.

\bibitem{Courant-Hilbert}
R.~Courant and D.~Hilbert, \emph{Methods of Mathematical Physics, vol.
  II}.\hskip 1em plus 0.5em minus 0.4em\relax New York: Interscience
  publishers, 1953, vol. 336.

\bibitem{gnedenko}
B.~Gnedenko, \emph{Theory of Probability}.\hskip 1em plus 0.5em minus
  0.4em\relax New York, NY: Chelsea, 1968.

\bibitem{Koopman1}
S.~L. Brunton, J.~L. Proctor, and J.~N. Kutz, ``Discovering governing equations
  from data by sparse identification of nonlinear dynamical systems,''
  \emph{Proceedings of the National Academy of Sciences}, vol. 113, no.~15, pp.
  3932--3937, 2016.

\end{thebibliography}


%
%





%

%

\onecolumn

\section*{APPENDIX : Detailed Proofs and Supplementary Material}


The supplementary materials contain detailed proofs of the results that are missing in the main paper.

\subsection{Proof of Proposition 1}
\begin{proof}
Due to Assumption \ref{basis}, we may write: 
\begin{equation}
\sum_{i=1}^{N_t} \alpha_t^i \phi^i(x) = \sum_{i=1}^ N c_t^i \phi^i(x) + \sum_{i=1}^{N_{t+1}} \alpha_{t+1}^i \sum_{j=1}^{N_t} \beta^{ij} \phi^j(x).
\end{equation}
Taking inner products on both sides of the above equation, we obtain:
\begin{equation}
    \sum_{i=1}^{N_t} \alpha_t ^i <\phi^i, \phi^k> = \sum_{i=1}^ {N_t} c_t^i <\phi^i, \phi^k>  +
    \sum_{i=1}^{N_{t+1}} \alpha_{t+1}^i \sum_j \beta^{ij} <\phi^j,\phi^k>, \, \forall\, k,
\end{equation}
where $<\phi^i, \phi^k> = \int \phi^i(x) \phi^k(x) p_t(x) dx$, i.e., the weighted inner product using the density $p_t(\cdot)$. Denote the Gram matrix $\mathcal{G}_t = [<\phi^i, \phi^k>]$, $i,k = 1 \cdots N_t$, and $\alpha_{t+1}^i = 0$, for all $i> N_{t+1}$. The above equations can then be written as:
\begin{equation}\label{Galerkin}
    \bar{\alpha}_t \mathcal{G}_t = \bar{c}_t \mathcal{G}_t + \bar{\alpha}_{t+1}B_t \mathcal{G}_t.
\end{equation}
Taking the inverse of $\mathcal{G}_t$ on both sides, the answer follows.
\end{proof}

\subsection{Proof of Proposition 2}
\begin{proof}
Owing to Assumption \ref{basis}, $\Phi_{t+1}^R = B_t \Phi_t^R$, and thus, $\Phi_{t+1}^R\Phi_t^{R'} = B_t \Phi_t^R \Phi_t^{R'}.$ Let $\Phi_t^R \Phi_t^{R'} \equiv G^R_t$. Then, the RL estimate can be written as:
\begin{align}
    \bar{\alpha}_t^R = \bar{c}_t G_t^R (G_t^R)^{-1} + \bar{\alpha}_{t+1}^R B_t G_t^R (G_t^R)^{-1} + V_t^R \Phi_t^{R'}(G_t^R)^{-1},
\end{align}
where $V_t^R = [v_t^{(1)}, \cdots v_t^{(R)}]$ is the noise sequence.
 Noting that the noise sequence is i.i.d., the covariance of $V^R_t$ is some $\sigma_v^2 I_R$, where $I_R$ is the $R$ dimensional identity matrix and $\sigma_v^2$ is the variance of the i.i.d. noise sequence. Thus, the covariance of the least squares estimate above is given by 
\begin{align*}
P_t^R &= \sigma_v^2 (G^R_t)^{-1}, ~
G^R_t = R \mathcal{G}^R_t,\\
\mathcal{G}^R_t &\equiv
\begin{bmatrix} \frac{1}{R}\sum_i \phi^1 (x_t^{(i)})\phi^1(x_t^{(i)}), \cdots, \frac{1}{R}\sum_i \phi^1(x_t^{(i)})\phi^N(x_t^{(i)})\\
\ddots \\ \frac{1}{R}\sum_i \phi^N (x_t^{(i)})\phi^1(x_t^{(i)}), \cdots, \frac{1}{R}\sum_i\phi^N(x_t^{(i)})\phi^N(x_t^{(i)}) \end{bmatrix}.
\end{align*}
We need to show that $\bar{\alpha}_t^R \rightarrow \bar{\alpha}_t$ in m.s. sense. We proceed by induction. Let us assume that $\bar{\alpha}_{t+1}^R \rightarrow \bar{\alpha}_{t+1}$ in m.s.s. Define $S = \frac{1}{R}\bar{\alpha}_{t+1}^R \Phi_{t+1}^R \Phi_t^{R'} (\mg_t^R)^{-1} - \bar{\alpha}_{t+1}B_t$. Then:\\
\begin{equation}
E||\bar{\alpha}_{t}^R  -\bar{\alpha}_t||^2 = E||S||^2 + E[S (V_t^R\Phi_t^{R'}(\Phi_t^R\Phi_t^{R'})^{-1})'] + \text{Tr} \frac{1}{R^2}E[(\mg_t^R)^{-1} \Phi_t^R V_t^{R'}V_t^R \Phi_t^{R'}(\mg_t^R)^{-1}],
\end{equation}
where $\text{Tr}$ denotes the trace of a matrix. Since $\bar{\alpha}_{t+1}^R \rightarrow \bar{\alpha}_{t+1}^R$ in m.s.s by assumption, and $\frac{1}{R} \Phi_{t+1}^R \Phi_t^{R'} = B_t\mg_t^R$,  $E||S||^2 \rightarrow 0$ as $R\rightarrow \infty$. The second term is zero since $V_t^R$ is independent of $\Phi_t^R$ and zero mean. For the last term, note that:
\begin{align}
    \text{Tr} \frac{1}{R^2}E[(\mg_t^R)^{-1} \Phi_t^R V_t^{R'}V_t^R \Phi_t^{R'}(\mg_t^R)^{-1}]
= \frac{1}{R^2}E_{\mg_t^R} [(\mg_t^R)^{-1} \Phi_t^R E[V_t^{R'} V_t^{R}] \Phi_t^{R'}(\mg_t^R)^{-1}| \mg_t^R]
= \frac{\sigma_v^2}{R} E_{\mg_t^R}[(\mg_t^R)^{-1}]
\end{align}
Noting that $\mg_t^R \rightarrow \mg_t$ almost surely, and $\mg_t$ is a finite matrix, it follows that $\frac{\sigma_v^2}{R} E_{\mg_t^R}[(\mg_t^R)^{-1}] \rightarrow 0$ as $R\rightarrow \infty$. Therefore, the above means that $\bar{\alpha}_t^R \rightarrow \bar{\alpha}_t$ in m.s.s. if $\bar{\alpha}_{t+1}^R \rightarrow \bar{\alpha}_{t+1}$ in m.s.s. A similar argument as above can be used to show that $\bar{\alpha}_T^R \rightarrow \bar{\alpha}_T$ in m.s.s. for the final time time $T$, and thus, by backwards induction the result follows for all $t$.\\
\end{proof}

\subsection{Proof of Corollary 1}
\begin{proof}
In general, since $N < N_t$, $\phi^i(f(x)) = \sum_{j=1}^{N_t} \beta^{ij} \phi^j(x) = \sum_{j=1}^N \beta^{ij} \phi^j(x) + \sum_{j=N+1}^{N_t} \beta^{ij} \phi^j(x)$. Thus,
$\Phi_{t+1}^R = B_t \Phi_t^R + \tilde{B}_t \tilde{\Phi}_t^R$, where
$ B_t = [\beta^{ij}], i,j =1,\cdots N,$ $\tilde{B}_t = [\beta^{ij}], i=1,\cdots N; j = N+1, \cdots N_t$, 
\begin{align}
    \Phi_t^ R = \begin{bmatrix} \phi^1(x_t^{(1)}), \cdots, \phi^1(x_t^{(R)})\\ \ddots\\ \phi^N(x_t^{(1)}), \cdots, \phi^N(x_t^{(R)}) \end{bmatrix}, ~ \text{and}, ~ \tilde{\Phi}_t^ R = \begin{bmatrix} \phi^{N+1}(x_t^{(1)}), \cdots, \phi^{N+1}(x_t^{(R)})\\ \ddots\\ \phi^{N_t}(x_t^{(1)}), \cdots, \phi^{N_t}(x_t^{(R)}) \end{bmatrix}.
\end{align}
Therefore, $\Phi_{t+1}^R \Phi_t^{R'} = B_t \Phi_t^R \Phi_t^{R'} + \tilde{B}_t \tilde{\Phi}_t^R \Phi_t^{R'}$, and the second term $\Delta \equiv \tilde{B}_t \tilde{\Phi}_t^R \Phi_t^{R'}$ on the right, in general, biases the coefficients $\alpha_t^i$, for $i = 1\cdots N$. Further, even if $\Delta =0$, which is feasible if the basis functions are orthonormal, i.e., $< \phi^i, \phi^j> = \delta_{ij}$, where $\delta_{ij}$ is the Kronecker delta function, the coefficients $\alpha_t^i$, $i> N$, are missed, which then affects the calculation of $J_{t-1}(\cdot)$, in particular, the coefficients $\alpha_{t-1}^i$, $i=1,\cdots N$. Thus, the solution, even in such a case, gets biased at the next time step $t-1$.\\
\end{proof}

\subsection{Proof of Theorem 1}
\begin{proof}
Recall that $
\Phi^R_t = \begin{bmatrix}
\phi^1(x^{(1)}_t) \cdots \phi^1(x^{(R)}_t)\\
\ddots\\
\phi^{N_t}(x^{(1)}_t) \cdots \phi^{N_t}(x^{(R)}_t)
\end{bmatrix}.$

We have $\mg^R_t = \mg_t + \mathcal{E}^R_t $, where 
\begin{align}
    \mg_t = \begin{bmatrix}
E[\phi^1]^2 \cdots E[\phi^1\phi^{N_t}]\\
\ddots\\
E[\phi^{N_t}\phi^1] \cdots E[\phi^{N_t}]^2\end{bmatrix} = \lim_R \frac{1}{R} \Phi^R_t\Phi_t^{R'},
\end{align}
is the Gram matrix for the basis functions $\phi^i$ at time $t$. 
Assuming that $R$ is large enough, we may write: 
\begin{align}
    (G^R_t)^{-1} = \frac{1}{R} (\mg_t + \mathcal{E}^R_t)^{-1} = \frac{1}{R} \sum_{k=0}^{\infty} (-1)^k \mg_t^{-1}(\mathcal{E}^R_t \mg_t^{-1})^k.
\end{align}
Thus, taking (spectral) norms on both sides, we obtain:
$||(G^R_t)^{-1}|| \leq \frac{1}{R} \sum_k ||\mg_t^{-1}||^{k+1}||\mathcal{E}^R_t||^k.$
Using the assumption that $||\mg_t^{-1}|| < C$, we obtain:
\begin{equation}\label{eq.1}
    ||(G^R_t)^{-1}|| \leq \frac{C}{R}\sum_{k=0}^{\infty} C^k||\mathcal{E}^R_t||^k.
\end{equation}
Next, we find a bound on $||\mathcal{E}^R_t||$. \\
Note that $\mathcal{E}^R_t = H_t \bar{\mathcal{E}}^R_t H'_t,$ where $\bar{\mathcal{E}}^R_t = \begin{bmatrix}
\bar{\epsilon}_2^R \cdots \bar{\epsilon}_{M_t+1}^R\\
\ddots\\
\bar{\epsilon}_{M_t+1}^R \cdots \bar{\epsilon}_{2M_t}^R,
\end{bmatrix}$
where $\bar{\epsilon}_p^R = E[x^p] - \frac{1}{R} \sum_{i=1}^R (x^{(i)})^p,$ i.e, the error between the moment $E[x^p]$ and its R-sample empirical estimate. Then, it is clear that $E[\bar{\epsilon}_p^R] = 0$, and $\text{var}[\bar{\epsilon}_p^R] = \frac{\text{var}(x^p)}{R}.$ Also note that $\text{var}[x^p] = E[x^{2p}] - (E[x^p])^2 = (2p-1)!! \sigma_X^{2p}$, if $p$ is odd, and $[(2p-1)!! - (p-1)!!^2]\sigma_X^{2p}$, if $p$ is even.\\
If $R$ is large enough, then owing to the Central Limit Theorem \cite{gnedenko}, $\bar{\epsilon}_p^R \sim \mathcal{N}(0, \frac{var[x^p]}{R})$. Further, given that $R$ is large enough, it is reasonable to assume that all $\bar{\epsilon}_p^R$, $p<2M_t$, have converged, and thus, the error matrix $\bar{\mathcal{E}}^R_t \approx \begin{bmatrix}
0 &\cdots& 0\\
\vdots & \vdots & \vdots\\
0 &\cdots& \bar{\epsilon}_{2M_t}^R
\end{bmatrix}.$
And thus, it follows that $\mathcal{E}^R_t = \bar{\epsilon}_{2N}^R H_{M_t} H^{N_t}$, which implies that:
\begin{equation}\label{eq.2}
    ||\mathcal{E}^R_t|| \leq |\bar{\epsilon}_{2M_t}^R| ||H_{M_t}||||H^{N_t}|| \leq C' |\bar{\epsilon}_{2M_t}^R|,
\end{equation}
using the assumption that $||H_{M_t}||||H^{N_t}|| \leq C'$.
Hence, using \eqref{eq.1} and \eqref{eq.2}, it follows that if
\begin{equation}\label{star}
CC'|\bar{\epsilon}_{2M_t}^R| \leq \beta < 1,
\end{equation}
then $||(G^R_t)^{-1}|| \leq \frac{C}{R(1-\beta)}$, i.e. $||P^R_t|| \leq \frac{\sigma_v^2 C}{R(1-\beta)}$ Thus, it suffices to have:
\begin{equation}
    R \geq \frac{\sigma_v^2 C}{ \delta(1-\beta)}, 
\end{equation}
for $||P^R_t|| \leq \delta.$
However, note that condition \eqref{star} still needs to be satisfied. \\
Next, using the tail bounds for a Gaussian random variable \cite{gnedenko}, we have that $Prob(|\bar{\epsilon}_{2M_t}^R| > \frac{n\sigma_{2M_t}}{\sqrt{R}}) \leq 2 e^{-n^2/2}$. Thus, if we require that $R$ is such that:
\begin{equation}\label{dstar}
\frac{n\sigma_{2M_t}}{\sqrt{R}} \leq \frac{\beta}{CC'},
\end{equation} 
then owing to the tail bounds, we have that: $Prob(||P^R_t|| \leq \delta) > 1- 2e^{-n^2/2}$. Thus, using the conditions \eqref{star} and \eqref{dstar}, if choose $R= \max[\frac{\sigma_v^2 C}{\delta(1-\beta)}, (\frac{nCC'}{\beta})^2\sigma_{2M_t}^2]$, the result follows.
\end{proof}

\subsection{Proof of Lemma 1}
\begin{proof}
It is clear that $\delta \chi_t^R \delta \chi_t^{R'} = R \mg^R$.
Next, $\delta x_{t+1}^{(i)} = \sum_k F_t^k (\delta x_t^{(i)})^k +  v_t^{(i)}$, and $\delta X_{t+1}^R\delta\chi_t^{R'} = [d^1, \cdots d^M]$, where $d^l \equiv \sum_{k=1}^M \sum_{i=1}^R F_t^k (\delta x_t^{(i)})^k (\delta x_t^{(i)})^l + \sum_{k>M} \sum_{i=1}^R F_t^k (\delta x_t^{(i)})^k (\delta x_t^{(i)})^l$. However, $\sum_{k>M} \sum_{i=1}^R F_t^k (\delta x_t^{(i)})^k (\delta x_t^{(i)})^l = R \Delta_t^{R,l}$. Thus, using the above identities, it follows that: $\delta X_{t+1}^R\delta \chi_t^{R'} = \mf_t (R\mg^R) + R \Delta_t^R + V_t^R \delta \chi_t^{R'}$, and substituting this into \eqref{MB-LS}, the result follows. 
\end{proof}

\subsection{Proof of Proposition \ref{LS-e}}

First, we need the following basic result.

\begin{lemma}\label{LS-e.1}
Let $h(z) = \sum_k H^k z^k$ be the Taylor series expansion of the function $h(\cdot)$, and let it be convergent for any finite $z$. Let the lowest power in the expansion be $L$. Then, given any $M \geq L$, and $\epsilon >0$, there exists a $Z$ such that if $z < Z$, $|\sum_{k> M} H^k z^k| \leq \epsilon|\sum_{k=1}^M H^k z^k|$. 
\end{lemma}

\begin{proof}
Let $S_M (z) = \sum_{k=1}^M H^k z^k$, and $\delta S_M (z) =\sum_{k> M} H^h z^k $. Since the series is convergent, given any $\epsilon >0$, there exists an $M^{\epsilon} (z) < \infty$, such that for all $M > M^{\epsilon} (z)$, $\delta S_M(z) \leq \epsilon S_M(z)$. Let $\bar{M}^{\epsilon}(z) = \max_{z'\leq z} M^{\epsilon}(z')$. That this value is finite is trivial. Then, it follows that if $M'$ is chosen such that $M'> \bar{M}^{\epsilon}(z)$, then for any $z'<z$, $\delta S_{M'}(z') \leq \epsilon S_{M'}(z')$. \\
Next consider a $Z$ such that $M^{\epsilon}(Z) > M$ for the first time. If this never happens, then $Z = \infty$, and the result trivially holds for the entire domain. Then, if $z<Z$, due to the above analysis, it follows that $\delta S_M (z) \leq \epsilon S_M(z)$, for all $z < Z$. The only case remaining is if $z=0$: this is never possible unless the lowest degree in the expansion is greater than $M$ which is precluded by the assumption of there being lower order terms than $M$. This concludes the proof of the result.\\
\end{proof}
Using the result above, we can now prove Proposition \ref{LS-e}.

\begin{proof}
Recall that $\mf_t\mg + \Delta_t = [f_t^j]$ where $f_t^j = E[\delta f \delta x_t^j]$, and $\delta f = \sum_{k>0} F_t^k \delta x_t^k$. Thus, it follows that $\tilde{f}_t^j = \sum_{k=1}^M F_t^k E[\delta x_t^{j+k}]$, and $\Delta_t^j = \sum_{k>M} F_t^k E[\delta x_t^{j+k}]$.\\
First, suppose that $j$ is odd, then: $\tilde{f}_t^j = \sum_{k=1}^M F_t^k C_{j+k} \sigma_X^{j+k}$ where $E[\delta x_t^p] = C_p \sigma_x^p$, for some constant $C_p$. And similarly $\Delta_t^j = \sum_{k>M} F_t^k C_{j+k} \sigma_X^{j+k}$. Thus, $f_t^j (\sigma_X) = E[\delta f \delta x_t^j] = \sigma_X^{j} \sum_k F_t^k C_{j+k} \sigma_X^k$. Note that if the input distribution is zero mean Gaussian, then we only have even powers of $\sigma_X$. The above is simply a Taylor series expansion of the function $f_t^j(\sigma_X^2)$. Thus, using Lemma \ref{LS-e.1}, it follows that there exists $(\sigma_X^j)^2 < \infty$ such that $|\Delta_t^j (\sigma^2)| \leq \epsilon |\tilde{f}_t^j (\sigma^2)| $, for all $\sigma^2 \leq (\sigma_X^j)^2$. A similar argument holds for even $j$.\\
Next, choose $\sigma_X^2 = \min_j (\sigma_X^j)^2$. Then, using the above result, it follows that:
$|\Delta_t^j| \leq \epsilon|\tilde{f}_t^j|$, for all $j = 1, ..M$, thereby proving the result.\\
\end{proof}

\subsection{Proof of Corollary 2}
\begin{proof}
Note that $\delta \chi_t^R \delta \chi_t^{R'} = R \mg^R$, and $\mf_t^R = \mf_t  + R\Delta_t^R(\frac{(\mg^R)^{-1}}{R}) + \frac{1}{R}V_t^R \delta \chi_t^{R'} (\mg^R)^{-1}$. Defining $S \equiv  \Delta_t^R (\mg^R)^{-1} - \Delta_t \mg^{-1}$, and noting that $\mf_t$, $S$ etc. are all row vectors, we can show that:
\begin{align}
    E||\mf_t^R - (\mf_t + \Delta_t \mg^{-1})||^2 
= E||S||^2 +\frac{2}{R} E[S(V_t^R \delta \chi_t^{R'} (\mg^R)^{-1})']
+\frac{1}{R^2} E[ V_t^R \delta \chi_t^{R'}(\mg^R)^{-1}(\mg^R)^{-1} \delta \chi_t^R V_t^{R'}].
\end{align}
Since $\mg^R \rightarrow \mg$, $\Delta_t^R \rightarrow \Delta_t$ in m.s.s. as $R\rightarrow \infty$, and $\mg, \Delta_t$ are finite matrices, $E||S||^2 \rightarrow 0$ as $R \rightarrow \infty$, while $E[S'V_t^R \delta \chi_t^{R'}(\mg^R)^{-1}] = 0$, since $V_t^R$ is a zero mean sequence, and independent from $\delta \chi_t^R$, and hence from $S$ and $\mg^R$. The last term can be written as:
\begin{align}
    \frac{1}{R^2} E[ V_t^R \delta \chi_t^{R'}(\mg^R)^{-1}(\mg^R)^{-1} \delta \chi_t^R V_t^{R'}] &= \frac{1}{R^2} \text{Tr}( E[(\mg^R)^{-1} \delta \chi_t^R V_t^{R'} V_t^{R} \delta \chi_t^{R'} (\mg^R)^{-1}])\\ &= \frac{1}{R^2}\text{Tr}( E_{\mg^R} E[(\mg^R)^{-1} \delta \chi_t^R V_t^{R'} V_t^{R} \delta \chi_t^{R'} (\mg^R)^{-1}| \mg^R]]) \\
    & = \text{Tr}( [E_{\mg^R} [\frac{\sigma_v^2}{R} (\mg^R)^{-1}]),
\end{align}
where $\text{Tr}(A)$ denotes the trace of a matrix $A$, and $E_{\mg^R}[.]$ represents the expectation with respect to $\mg^R$. Next, we know that $\mg^R \rightarrow \mg$ almost surely as $R\rightarrow \infty$, and since $\mg$ is a finite matrix, it follows that $\text{Tr}( [E_{\mg^R} [\frac{\sigma_v^2}{R} (\mg^R)^{-1}])
 \rightarrow 0$ as $R \rightarrow \infty$. Therefore, $E||\mf_t^R - (\mf_t + \Delta_t \mg^{-1})||^2 \rightarrow 0 $ as $R \rightarrow \infty$, proving the result.
\end{proof}

\subsection{Model Free Solution of the PPE}
In the model-free method, rather than estimating $\mf_t$ and substituting into $\mathcal{B}_t$ in \eqref{PPE}, one can directly try to estimate the Taylor coefficients $\mk_t$ given those at the next time $\mk_{t+1}$. A Least squares estimate of the same is easily seen to be:
\begin{equation}
    \mk_t^R = \mathcal{C}_t + \mk^R_{t+1} (\delta \chi_{t+1}\delta \chi_t')(\delta \chi_t\delta\chi_t')^{-1},
\end{equation}
where note that unlike in the model-based case where we use $\delta X_{t+1}$, we use\\ $\delta \chi_{t+1}^R = \begin{bmatrix}
\delta x_{t+1}^{(1)} &\cdots& \delta x_{t+1}^{(R)} \\
\vdots & \vdots &\vdots\\
(\delta x_{t+1}^{(1)})^M & \cdots & (\delta x_{t+1}^{(R)})^M
\end{bmatrix}$. 
In essence, the correlation $\delta \chi_{t+1}\delta \chi_t'$ allows one to deduce the matrix $\mathcal{B}_t$ in the PPE equation \eqref{PPE} rather than using the known structure of $\mathcal{B}_t$. 

The accuracy and variance of the solution are amenable to an identical analysis as we did for the model-based case, and it should not come as a surprise that even in this case, it is determined by the convergence and variance of the empirical Gram matrix $\mg^R$.\\

\noindent \textbf{Model-based or Model-free?}\\
The key to computing the solution is doing the product $\delta \chi_{t+1}\delta \chi_t$, the number of computations are $\mathcal{O}(2R^2M^2)$, while that in the model-based case  for the product $\delta X_{t+1}\delta \chi_t'$ is $\mathcal{O}((R^2+R)M^2)$. Since $R$ is typically going to be large, this shows that we need to do approximately $2 \times$  the computation to get the answer in the model-free method, since it is trying to infer the matrix $\mathcal{B}_t$ rather than using the inherent structure. Further, given the same number of computations, we would be able to do $\approx 2R$ simulations to the $R$ of the model-free method, and thus, the variance in the empirical Gram matrix $\mg^R$ is $\frac{1}{2}$ times that of the model-free empirical Gram matrix. Again, this shows that we should prefer the model based approach, either to save computation or reduce the variance of the solution.

\section{The Discounted Infinite Horizon Case}
Thus far, we have considered the finite horizon optimal control case. Now, we consider the more prevalent discounted infinite horizon case. Again, we consider the deterministic case first, and then the stochastic case. It is always possible to turn any discounted infinite horizon case into a finite horizon problem with a sufficiently long horizon, and a zero terminal cost. In this section, we concentrate on the infinite horizon policy evaluation problem, given the dynamics $x_{t+1}= f(x_t)$. In essence, we want to evaluate the cost of the above dynamics in the discounted case. To that end, let $c(x)$ denote the incremental cost and let $\beta < 1 $ denote a discount factor.
the discounted cost-to-go given an initial condition $x_0$ is given by:
$J(x_0) = \sum_{t=0}^{\infty} \beta^t c(x_t)$,
where $x_t$ evolves according to the dynamics above. The cost function $J$ obeys the following stationary equation, given any state $x$:
\begin{equation}\label{PE_beta}
    J(x) = c(x) + \beta J(f(x)). 
\end{equation}
Next, consider a nominal trajectory, given some initial condition $x_0$, say $\bar{x}_{t+1} = f(\bar{x}_t)$. Next, let us do expansions of the cost about this nominal trajectory, as we did previously, $J(x_t) = J(\bar{x}_t + \delta x_t) = J(\bar{x}_t) + K_t^1 \delta x_t + K_t^2 \delta x_t^2 + \cdots$, $c(x_t) = c(\bar{x}_t + \delta x_t) = c(\bar{x}_t) + C_t^1 \delta x_t + C_t^2 \delta x_t^2+ \cdots$, and 
\begin{equation}
    J_{t+1}(f(x_t)) = J_{t+1}(f(\bar{x}_t) + \delta f_t) = J_{t+1}(f(\bar{x}_t)) + K_{t+1}^1 \delta f_t + K_t^2 \delta f_t^2 + \cdots,
\end{equation}
where $\delta f_t = F_t^1 \delta x_t + F_t^2 \delta x_t^2 + \cdots$. Substituting into \eqref{PE_beta}, and equating the different powers of $\delta x_t$ on both sides, we obtain:
\begin{align} \label{PPE_beta}
    \bar{J}_t = \bar{c}_t &+ \beta \bar{J}_{t+1}, \nonumber\\
    K_t^1 = C_t^1 &+ \beta K_{t+1}^1 F_t^1, \nonumber\\
    K_t^2 = C_t^2 &+ \beta(K_{t+1}^1 F_t^2 + K_{t+1}^2 (F_t^1)^2), \\
    &\vdots \nonumber
\end{align}
where $\bar{J}_t = J(\bar{x}_t)$, $\bar{c}_t = c(\bar{x}_t)$, and $\bar{J}_{t+1} = J_{t+1}(f(\bar{x}_t))$, with terminal condition $J_T (x) = 0$ for any $x$. Again, note the beautiful perturbation structure of the above equations in that the higher order terms do not affect the calculations of the lower order terms, and thus, we may close our computations at any desired order. \\
However, one has to be careful about using the above perturbation equations. An implicit assumption here is that $T$ is large enough such that $J_t(x)$ becomes stationary, i.e., converges to some $J_{\infty}(x)$, $J_t(x) \approx J_{\infty}(x)$, for all $ t < \bar{t}$ where $T-\bar{t} \ll T$.  Moreover, we should use the solution to the above equations \eqref{PPE_beta} only after the initial transient time $\bar{t}$ has passed, i.e., we use it only $J_t$ for $t< \bar{t}$. The situation is illustrated in Fig. \ref{PE_discount}. Further, we should note that albeit $J_{\infty}(x)$ is time invariant, when we expand it locally around a nominal trajectory, the result expansions are time-varying (again, see Fig. \ref{PE_discount}). \\
\begin{figure}[h!]
    \vspace{-0.8in}
    \centering
    \includegraphics[width=10cm,height=12cm,keepaspectratio]{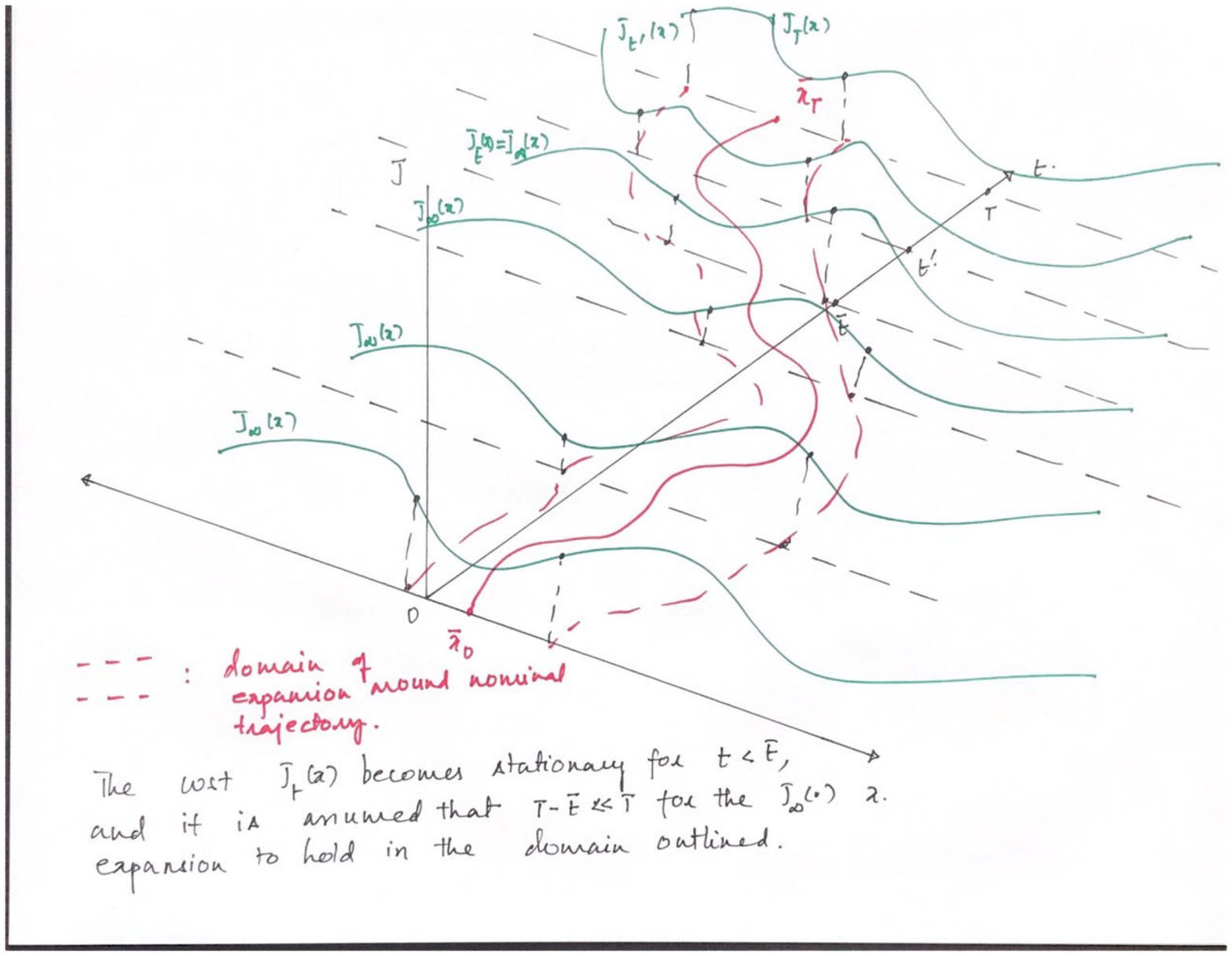}   
    \vspace{-1in}
    \caption{The Discounted Infinite Horizon Case}
    \label{PE_discount}
\end{figure}
Moreover, just like in the finite horizon case, One may also choose not to do a time varying decomposition about a trajectory, and instead do an expansion about the time invariant trajectory, $\bar{x}(t) = 0$, assuming that $f(0) = 0$. However, identical to the situation in the finite horizon case, most initial conditions $x_0$ would be far from the origin, and thus, we would need far more terms to make an accurate approximation when compared to the time varying case, thereby drastically increasing the computational burden of the method. \\
The policy evaluation can also be done in an RL fashion as was done previously for the finite horizon case in an almost identical fashion by estimating $\mf_t$, or directly by using $\delta \chi_{t+1}$, the only caveat again being that we use the solution only after it has equilibriated, i.e., for $t < \bar{t}$. The stochastic case remains intractable owing again to the lack of a perturbation structure, and thus, any RL solution is bound to have high variance, and thus, be inaccurate.\\

\begin{remark}
\textit{The Method of Characteristics.} The equation \ref{pol_eval}, when written in continuous time becomes the PDE: $\frac{\partial J}{\partial t} + c + f \frac{\partial J}{\partial x} = 0$, where the dynamics are now in continuous time, i.e., $\dot{x} = f(x)$ and the cost function is given by: $\int_0^T c(x_t)dt$. The classical method of characteristics reduces the above PDE into a family of ordinary differential equations (ODE) called the characteristic ODEs/ Lagrange-Charpit equations, in terms of the state $x$ and the co-state $q = \frac{\partial J}{\partial x}$, given terminal conditions $\bar{x}_T$ and $\bar{q}_t = \frac{\partial g}{\partial x}|_{\bar{x}_T}$ \cite{Courant-Hilbert} (see Section VII B). The PPE equations \eqref{PPE} are the discrete time analogs of the perturbation expansion of the characteristic ODEs about the nominal characteristic curve $(\bar{x}_t, \bar{q}_t)$. Next, consider the infinite horizon equation \eqref{PE_beta}: this is the discrete time analog of the PDE $\frac{\partial J}{\partial t} = c + \beta f \frac{\partial J}{\partial x}$, where the cost function now is given by $\int_0^{\infty} c(x_t) e^{-\beta t} dt$. In this case, the  perturbation equations \eqref{PPE_beta} are the discrete time analogs of the perturbation expansions of the characteristic ODEs around the nominal path $\dot{\bar{x}} = f(\bar{x})$. In this case, an additional condition is that after a suitable amount of time, the solution equilibriates, i.e., $\frac{\partial J_{\infty}}{\partial t} =  c+ \beta f\frac{\partial J_{\infty}}{\partial x} = 0$. The stationary PDE is difficult to solve and typically is solved by evolving the time varying PDE till the solution becomes stationary, which is exactly what is done when using \eqref{PPE_beta} in discrete time.\\
\end{remark}

\begin{remark}
\textit{Koopman Operator.} In recent years, there has been an increasing interest in the Koopman operator approach to the study of nonlinear systems \cite{Koopman1}. The idea is to consider the recursive dynamic map $\mk_f \cdot g \equiv g(f(x))$, where $g(\cdot)$ is termed an observable and to solve for this recursive map in a data based fashion. We note that the Koopman recursive map is equivalent to the policy evaluation equation \eqref{pol_eval}, with the cost $c(x) = 0$, and the terminal condition given by $g(x)$. Thus, any data based solution to the Koopman problem also suffers from the same variance and convergence issues as those afflicting policy evaluation, when it is solved in an RL/ data based fashion.
\end{remark}



 

 


\end{document}